%% file: main.tex
\documentclass{article}

\PassOptionsToPackage{numbers, compress}{natbib}

\usepackage[preprint]{neurips_2026}

\usepackage[utf8]{inputenc} 
\usepackage[T1]{fontenc}    
\usepackage[table]{xcolor}  
\usepackage[
  colorlinks=true,
  citecolor=green!50!black,
  linkcolor=red,
  urlcolor=blue
]{hyperref}
\usepackage{url}            
\usepackage{booktabs}       
\usepackage{amsfonts}       
\usepackage{nicefrac}       
\usepackage{microtype}      
\usepackage{amsmath}
\usepackage{algorithm}
\usepackage{graphicx}
\usepackage{subcaption} 
\usepackage{xspace}
\usepackage{enumitem}
\usepackage{algpseudocode}
\usepackage{amsthm}
\usepackage{amssymb} 
\usepackage{tcolorbox}
\usepackage{tabularx}
\usepackage{caption}
\usepackage{array}
\usepackage{fontawesome5}
\usepackage{tikz}
\usepackage{wrapfig}
\newcolumntype{Y}{>{\centering\arraybackslash}X}
\newcolumntype{C}{>{\centering\arraybackslash}X}
\newcolumntype{Z}{>{\columncolor{zeroshotbg}\centering\arraybackslash}X}
\newcolumntype{S}{>{\columncolor{sftbg}\centering\arraybackslash}X}
\newcolumntype{D}{>{\columncolor{wsdmbg}\centering\arraybackslash}X}

\definecolor{lightgreen}{rgb}{0.85, 1.0, 0.85}
\definecolor{best}{RGB}{210, 230, 250} 
\definecolor{second}{RGB}{235, 245, 255} 
\definecolor{wmsdblue}{RGB}{0,55,130}
\definecolor{zeroshotbg}{RGB}{235,245,255}
\definecolor{sftbg}{RGB}{245,238,255}
\definecolor{wsdmbg}{RGB}{235,250,240}
\definecolor{rulegray}{RGB}{190,190,190}

\newtheorem{theorem}{Theorem}

\newtheorem{proposition}[theorem]{Proposition}

\theoremstyle{definition}

\theoremstyle{remark}

\DeclareRobustCommand{\method}{\textit{WMSD}\xspace}
\DeclareRobustCommand{\methodname}{\textit{World-Model Self-Distillation}\xspace}
\DeclareRobustCommand{\benchname}{\textit{WorldTasks-Bench}\xspace}
\DeclareRobustCommand{\datasetname}{\textit{WorldTasks}\xspace}

\definecolor{unibernred}{RGB}{238,16,30}
\newlength{\unibernLogoHeight}

\newcommand{\unibernWordLine}[3]{%
  \node[
    anchor=base west,
    inner sep=0pt,
    text=unibernred,
    font=\fontfamily{phv}\bfseries\fontsize{120}{120}\selectfont,
    xscale=1.09
  ] at (#1,#2) {#3};
}

\newcommand{\unibernTikZLogo}[1][0.75cm]{%
  \begingroup
  \setlength{\unibernLogoHeight}{#1}%
  \pgfmathsetmacro{\unibernScale}{\the\unibernLogoHeight/335pt}%
  \begin{tikzpicture}[x=1pt,y=1pt,scale=\unibernScale,transform shape]
    \path[use as bounding box] (0,0) rectangle (1012,335);

    \node[
      anchor=base west,
      inner sep=0pt,
      text=black,
      font=\fontfamily{ptm}\bfseries\itshape\fontsize{240}{240}\selectfont,
      xscale=1.02,
      yscale=0.98
    ] at (34,61) {u};

    \node[
      anchor=base west,
      inner sep=0pt,
      text=black,
      font=\fontfamily{ptm}\itshape\fontsize{122}{122}\selectfont,
      xscale=0.86,
      yscale=1.00
    ] at (190,172) {b};

    \unibernWordLine{305}{130}{UNIVERSITY}
    \unibernWordLine{305}{10}{OF BERN}
  \end{tikzpicture}%
  \endgroup
}

\newcommand{\unibernTikZLogoHor}[1][0.75cm]{%
  \begingroup
  \setlength{\unibernLogoHeight}{#1}%
  \pgfmathsetmacro{\unibernScale}{\the\unibernLogoHeight/335pt}%
  \begin{tikzpicture}[x=1pt,y=1pt,scale=\unibernScale,transform shape]
    \path[use as bounding box] (0,0) rectangle (1900,335);

    \node[
      anchor=base west,
      inner sep=0pt,
      text=black,
      font=\fontfamily{ptm}\bfseries\itshape\fontsize{240}{240}\selectfont,
      xscale=1.02,
      yscale=0.98
    ] at (34,61) {u};

    \node[
      anchor=base west,
      inner sep=0pt,
      text=black,
      font=\fontfamily{ptm}\itshape\fontsize{122}{122}\selectfont,
      xscale=0.86,
      yscale=1.00
    ] at (190,170) {b};

\node[
  anchor=base west,
  inner sep=0pt,
  text=unibernred,
  font=\fontfamily{phv}\bfseries\fontsize{165}{165}\selectfont,
  xscale=1.02
] at (280,55) {University of Bern};
  \end{tikzpicture}%
  \endgroup
}
\title{World Model Self-Distillation: Training World Models to Solve General Tasks}

\author{
Sebastian Stapf \quad
Pablo Acuaviva Huertos \quad
Aram Davtyan \quad
Paolo Favaro \\
Department of Computer Science\\[-0.45em]
\unibernTikZLogo[0.65cm]\\[-0.2em]
\texttt{
\{sebastian.stapf, pablo.ahuertos, aram.davtyan, paolo.favaro\}@unibe.ch
} \\[0.2em]
{\color{blue}
\faGlobe\ 
\href{https://sebastian-stapf.github.io/world-model-self-distillation/}{Project Page}
\quad
\faGithub\ 
\href{https://github.com/sebastian-stapf/world-model-self-distillation}{Code}
\quad
\faDatabase\ 
\href{https://huggingface.co/datasets/sebastian-stapf/WorldTasks}{Dataset}
}
}

\begin{document}

\maketitle

\begin{figure}[h]
    \centering
    \includegraphics[width=0.99\textwidth]{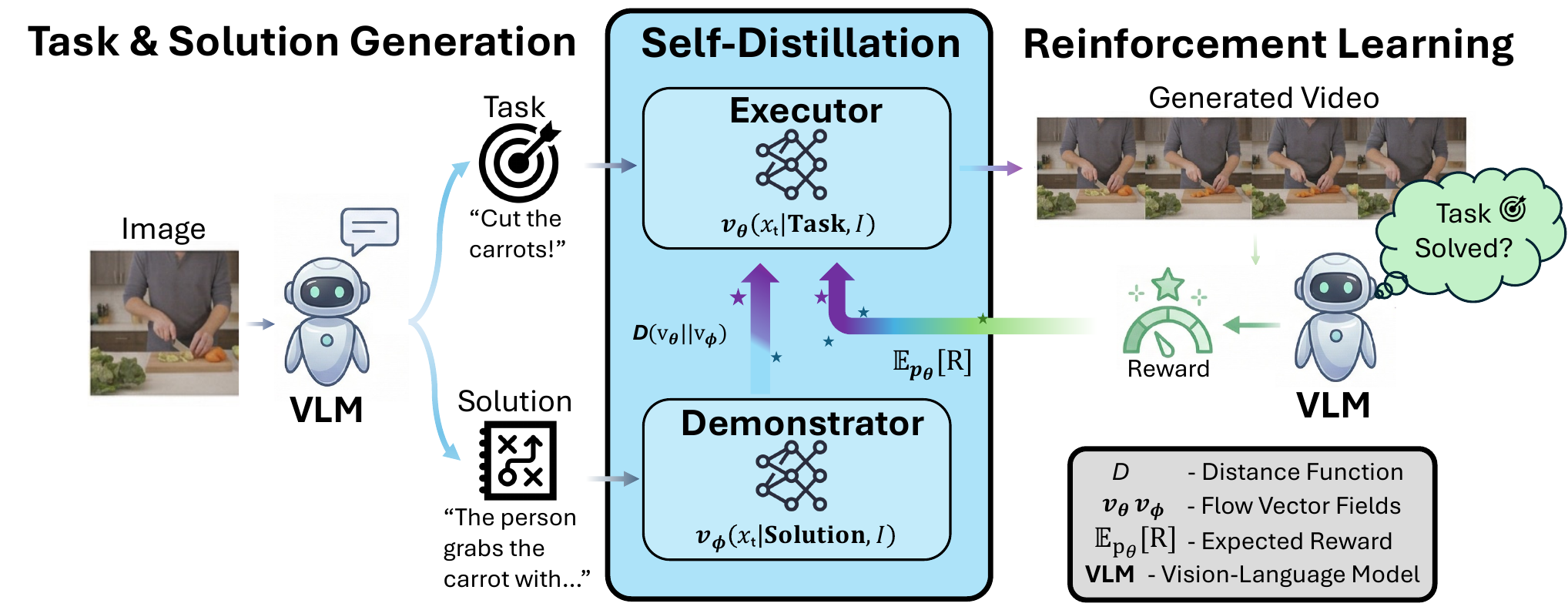}
    \caption{
     Overview of WMSD\@. The method addresses general tasks via a two-stage pipeline. \textbf{(Left)} A vision--language model (VLM) generates task descriptions along with corresponding solution prompts. \textbf{(Bottom $\rightarrow$ Top)} These solutions supervise the distillation of a video diffusion model (\textbf{Demonstrator}) into a task-conditioned video model (\textbf{Executor}), enabling the Executor to reproduce effective reasoning strategies. \textbf{(Right)} To further improve performance, reinforcement learning is applied: the VLM evaluates generated solution videos and provides feedback to refine the Executor.
    }
    \label{fig:teaser}
\end{figure}

\begin{abstract}
Pretrained video generators are promising visual world models that exhibit emergent task-solving abilities; however, their reliance on detailed textual descriptions limits their direct use for planning and decision-making. Existing approaches either outsource this reasoning to language or vision-language models, or rely on supervised fine-tuning with paired task-execution videos, which are costly to collect and difficult to scale. We propose a scalable framework that elicits task-solving ability in such models by combining self-distillation with reinforcement learning. Given an unlabeled scene image, a vision-language model generates a candidate task and a detailed step-by-step solution. The solution conditions a pretrained video diffusion model, the \emph{Demonstrator}; we distill its behavior into an \emph{Executor} conditioned only on the image and a short task prompt. This transfers execution knowledge from caption-guided generation to instruction-conditioned task solving without curated task-video supervision. We further improve the Executor with reinforcement learning from VLM feedback, exploiting the asymmetry between judging whether a sampled video satisfies a task and generating the solution. Experiments on \benchname and the DreamGen robotics benchmark show that the Executor surpasses the Demonstrator under our VLM-based evaluation protocol and transfers competitively to robotic tasks.
\end{abstract}

\input{sections/01_introduction}
\input{sections/02_related_work}
\input{sections/03_method}

\input{sections/04_experiments}

\section{Conclusion}

In summary, the experiments show that \method consistently improves task-solving ability, agent correctness, and physical consistency across a wide range of settings. A key strength of the framework is that it converts the detailed execution knowledge available to caption-guided video generation into a compact instruction-following interface, without requiring curated task-execution videos. In particular, combining on-policy self-distillation with reinforcement learning proves especially effective, enabling the model to surpass the Demonstrator under the VLM-based benchmark while maintaining efficient inference. The VLM-based reward further lets the model exploit the asymmetry between generating a correct future and recognizing one, turning noisy task-level feedback into measurable gains. At the same time, the Demonstrator anchor preserves useful pretrained behavior and prevents reinforcement learning from drifting toward visually implausible solutions. Beyond controlled benchmarks, the competitive transfer to robotic tasks further highlights the robustness and generality of the approach, suggesting that \method is a promising direction for scalable and data-efficient world model training.

\section*{Acknowledgements}
This work was supported by a grant from the Swiss National Supercomputing Centre (CSCS) under project ID a144 on Alps as part of the Swiss AI Initiative, and by SNSF Grant 10001278. Additional computations were carried out on UBELIX (https://www.id.unibe.ch/hpc), the high-performance computing cluster at the University of Bern.

\bibliographystyle{plainnat}

\bibliography{references}

\appendix

\input{sections/appendix}

\end{document}

%% file: sections/01_introduction.tex
\section{Introduction}

World models are a promising paradigm for enabling agents to reason about their environment by internally simulating possible action sequences and selecting those that best achieve a desired goal~\cite{Ha2018WorldModels,Sutton1991Dyna,Hafner2019PlaNet,Hafner2020Dreamer,Schrittwieser2020MuZero}. Recent advances in visually pretrained world models, particularly video generative models, have demonstrated striking emergent capabilities that resemble intelligent behavior~\cite{Guo2025AreVM,Wiedemer2025VideoMA,Acuaviva2025FromGT,Acuaviva2025RethinkingVI,Bruce2024Genie}.

Common instantiations of such world models are pretrained text- or image-to-video generators~\cite{Ho2022ImagenVideo,Blattmann2023SVD,He2025PreTrainedVG,Hassan2024GEMAG,Hong2025RELICIV,stapf2026composition}. However, their reliance on textual conditioning, typically requiring a detailed description of the scene or action, limits their direct applicability to task solving. In practice, they do not autonomously infer how to execute a task; instead, they depend on the reasoning of external models such as language models or vision-language models (VLMs) to specify the solution. Ideally, we would like the world model to be able to accept a high-level task description and internally generate a plausible sequence of actions, thereby directly leveraging the knowledge acquired during pretraining.

One direct way to close this gap is supervised fine-tuning: collect pairs of task instructions and videos that demonstrate successful executions, and train the video model to generate the corresponding trajectory. However, this approach requires a large and diverse set of successful demonstrations, covering many environments, objects, and levels of task abstraction. Acquiring such data is costly, especially when tasks are long-horizon or when success depends on fine-grained object interactions. Large-scale world-model platforms, robot-learning datasets, and video curation pipelines reduce this burden, but they do not remove the need for scalable task supervision~\cite{Agarwal2025CosmosWF,Chi2023DiffusionPolicy,Brohan2023RT1,Brohan2023RT2,ONeill2024OpenXEmbodiment,NVIDIA2025GROOTN1}.

Reinforcement learning offers a complementary route. Instead of imitating only fixed demonstrations, a model can sample candidate solutions, receive feedback, and improve the probability of generations that satisfy the task. This paradigm has been central to preference-based training of language models~\cite{Christiano2017DeepRL,Stiennon2020SummarizeHF,Ouyang2022TrainingLM}, and recent work has begun to adapt RL objectives to diffusion and flow-based generative models~\cite{Black2024TrainingDM,Fan2023DPOK,Wallace2023DiffusionDPO,Liu2025FlowGRPO}. In the video domain, however, this strategy faces a severe computational bottleneck. The most successful video generators are commonly based on diffusion or flow matching, and producing even a short clip may require many denoising or integration steps~\cite{Ho2020DDPM,Lipman2022FlowMF,Ho2022ImagenVideo,Blattmann2023SVD}. Since RL requires many rollouts per update, naively applying RL to multi-step video generators is prohibitively expensive.

Few-step distillation helps address this bottleneck~\cite{Salimans2022ProgressiveDistillation,Song2023ConsistencyModels,Luo2023LCM,Liu2022RectifiedFlow}. Distribution Matching Distillation (DMD) trains a fast student to match a slower diffusion teacher by minimizing an approximate distributional divergence between student and teacher samples~\cite{Yin2023OneStepDW}. Because the objective can be evaluated on student-generated samples, it is attractive for iterative improvement without paired real videos at every update~\cite{Agarwal2023OnPolicyDO,Yin2025CausVid}.

We argue that a similar framework can be leveraged beyond efficiency gains and used instead to elicit task-solving capabilities in video world models. First, by conditioning the student model, which we call the \emph{Executor}, on high-level task instructions (e.g., ``cut the carrots'') together with an initial observation, and training it to match outputs from a teacher, the \emph{Demonstrator}, which is conditioned on detailed execution descriptions, the student learns to map instructions directly to plausible action sequences. This effectively transforms the generator into an instruction-following, task-solving world model. Instances of such task-solution pairs are given in Fig.~\ref{fig:worldtasks_examples}.
Because this approach operates in a self-distillation setting~\cite{Hinton2015Distilling,Furlanello2018BornAgain}, it remains constrained by the task-solving ability of the demonstrator, effectively placing an upper bound on performance. To move beyond this limitation, reinforcement learning is introduced into the process. Generated rollouts can then be evaluated by a VLM, which assesses whether the produced video successfully fulfills the given instruction.

This relies on a generation-verification asymmetry: for many structured tasks, finding a valid solution can be much harder than checking a proposed one~\cite{song2025mind}. In our setting, we instantiate this verifier with a vision-language model, following work showing that VLMs can serve as zero-shot reward models for language-specified visual tasks~\cite{Rocamonde2023VisionLanguageMA,Wang2024RLVLMF,Jiang2025DistributionMD}. 
Nevertheless, raw VLM rewards can be noisy and inconsistent, especially for ambiguous visual tasks~\cite{Amodei2016ConcreteProblems,Huang2023VBench,Bansal2024VideoPhy}. We therefore view VLM feedback not as a standalone ground-truth reward, but as a weak verification signal to be combined with distributional regularization from the teacher. The combination with self-distillation provides a natural way to stabilize this signal. We call our method \emph{World Model Self-Distillation} (WMSD) and give a general overview in Fig.~\ref{fig:teaser}.

\begin{figure}[t]
    \centering
    \includegraphics[width=\linewidth]{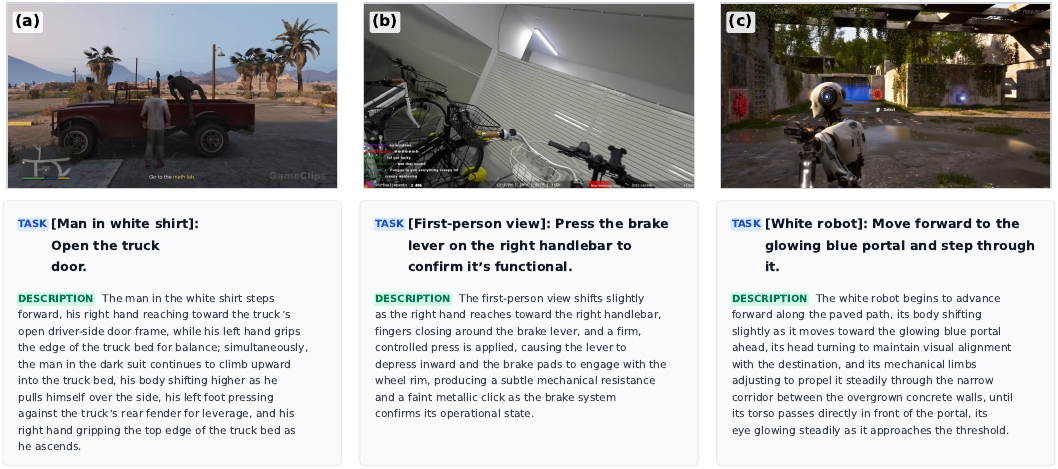}
    \caption{
    \textbf{WorldTasks examples.}
    Each panel shows an initial frame together with the addressed-agent task prompt and the original generated solution description.
    Examples cover human, first-person, and robot agents across interaction, manipulation, and navigation tasks.
    }
    \label{fig:worldtasks_examples}
\end{figure}

\begin{tcolorbox}[colback=second, colframe=second, width=\textwidth, boxrule=0pt, arc=0pt] 
To summarize, our \textbf{main contributions} are:
\begin{enumerate}
    \item We propose a self-distillation method that turns pretrained caption-conditioned video diffusion models into instruction-conditioned task solvers, without requiring paired task-execution videos.
    \item We augment this distillation procedure with reinforcement learning from VLM feedback, allowing the task-executing model to surpass its teacher under our VLM-based evaluation and remain competitive with methods trained using curated task-specific supervision.
   \item We provide a task-solution prompt dataset that leverages VLMs to derive tasks and detailed execution descriptions from unlabeled scene images.
    \item We provide a benchmark for evaluating general task solving in generated videos.

\end{enumerate}
\end{tcolorbox}

%% file: sections/02_related_work.tex
\section{Related Work}

\paragraph{Task-Conditioned World Models}

Prior work conditions world models on language, actions, or task specifications, including large-scale video foundation models and diffusion world models~\cite{Agarwal2025CosmosWF,Ho2022ImagenVideo,Blattmann2023SVD,stapf2026composition}, generative interactive environments~\cite{Bruce2024Genie}, and robot or visuomotor policies trained from task demonstrations~\cite{Chi2023DiffusionPolicy,Brohan2023RT1,Brohan2023RT2,ONeill2024OpenXEmbodiment}. Other planning systems combine video generation with language or vision-language models~\cite{Du2023LearningUP,Du2023VideoLP,Pan2024VLPVL}. We focus on a weaker inference-time interface: the Executor receives only an image and a short task instruction, while privileged step-by-step descriptions are used only during training through the Demonstrator.

\paragraph{Self-Distillation and Distribution Matching}

Knowledge distillation and self-distillation train compact or improved students from teacher predictions~\cite{Hinton2015Distilling,Rusu2015PolicyD,Furlanello2018BornAgain}. On-policy self-distillation and iterative refinement can improve models under distribution shift, especially when combined with reinforcement learning~\cite{Agarwal2023OnPolicyDO,Hubotter2026ReinforcementLV,Shenfeld2026SelfDistillationEC}. Distribution matching and consistency-style objectives similarly align student and teacher generative distributions, often for efficiency and stability~\cite{Yin2023OneStepDW,Salimans2022ProgressiveDistillation,Song2023ConsistencyModels,Luo2023LCM,Liu2022RectifiedFlow}. We use this asymmetry for task transfer rather than only acceleration: the teacher sees detailed execution descriptions, whereas the student must solve the task from a compact instruction.

\paragraph{Reinforcement Learning for Flow Models}

Recent methods adapt policy optimization to diffusion and flow-based generators and improve training stability through flow-specific refinements~\cite{Black2024TrainingDM,Fan2023DPOK,Wallace2023DiffusionDPO,Liu2025FlowGRPO,He2025TempFlowGRPOWT,Xue2025DanceGRPOUG,Li2025MixGRPOUF}. In contrast to reward-only alignment, we combine VLM task rewards with a Demonstrator-derived distillation reward and anchor loss, so RL improves task success while teacher guidance regularizes visual dynamics.

%% file: sections/03_method.tex
\section{Method}
\label{sec:method}

\paragraph{World-model setting}
We study instruction-conditioned video world models that generate future trajectories conditioned on an initial observation and a task specification. Given an observation $\mathcal{I}$ and instruction $\mathcal{T}$, the goal is to model
\begin{align}
    p(\tau \mid \mathcal{I}, \mathcal{T}),
\end{align}
where $\tau = \{x_t\}_{t\in[0,1]}$ denotes a latent video trajectory corresponding to task execution. We instantiate the world model using conditional flow-matching video generators~\cite{Lipman2022FlowMF} and consider a teacher--student setup inspired by knowledge and policy distillation~\cite{Hinton2015Distilling,Rusu2015PolicyD}: the teacher receives a detailed execution description $\mathcal{D}$, while the student must solve the task using only the instruction $\mathcal{T}$.

\paragraph{Setup}

We use conditional flow-matching video models~\cite{Lipman2022FlowMF} in a teacher--student setting. 
Each example contains an initial observation $\mathcal{I}$, a short task instruction $\mathcal{T}$, and a detailed execution description $\mathcal{D}$. The student, or \emph{Executor}, is conditioned only on
$c_{\mathrm{E}} = (\mathcal{I}, \mathcal{T})$, whereas the teacher, or \emph{Demonstrator}, is conditioned on the richer description $c_{\mathrm{D}} = (\mathcal{I}, \mathcal{D})$.
The teacher is fixed with parameters $\theta'$, while the student has trainable parameters $\theta$.

Let $x_t \in \mathbb{R}^d$ be the latent video state at flow time $t \in [0,1]$, with $x_0 \sim p_0$, where $p_0$ is the Normal distribution, and $x_1 \sim p_1$, where $p_1$ denotes the latent video data distribution. At inference time, $x_1$ is decoded into the generated video. The student and teacher define velocity fields $v_{\theta}(x_t,t \mid c_{\mathrm{E}})$ and $v_{\theta'}(x_t,t \mid c_{\mathrm{D}})$. A student flow trajectory satisfies
\begin{align}
    \frac{d x_t}{d t}
    =
    v_{\theta}(x_t,t \mid c_{\mathrm{E}}),
    \qquad
    x_0 \sim p_0.
    \label{eq:ODE}
\end{align}

\begin{wrapfigure}[12]{l}{0.45\textwidth}
\vspace{-1.5em}
\begin{minipage}{0.45\textwidth}
\hrule
\vspace{0.5ex}
\refstepcounter{algorithm}\label{alg:wmsd}
\textbf{Algorithm~\thealgorithm} World Model Self-Distillation
\vspace{0.5ex}
\hrule
\vspace{0.5ex}

\begin{algorithmic}[1]
\Require Demonstrator $v_{\theta'}$, Executor $v_\theta$, VLM $g$
\For{training iteration}
    \State Sample image $\mathcal I$
    \State Generate $(\mathcal T,\mathcal D)\leftarrow g(\mathcal I)$
    \State Sample rollout $\tau \sim p_\theta(\cdot\mid\mathcal I,\mathcal T)$
    \State Compute $r_{\mathrm{task}}$ using VLM feedback
    \State Compute $r_{\mathrm{distill}}$ via Eq.~\ref{eq:distillation_reward}
    \State Form reward $R(\tau)$ via Eq.~\ref{eq:total_reward}
    \State Optimize $\mathcal L_{\mathrm{RL}}$
    \State Compute $\mathcal L_{\mathrm{anchor}}$ via Eq.~\ref{eq:anchor_loss}
    \State Update using Eq.~\ref{eq:final_objective}
\EndFor
\end{algorithmic}

\vspace{0.5ex}
\hrule
\end{minipage}
\vspace{-1em}
\end{wrapfigure}
Teacher trajectories are analogous, replacing $v_\theta, c_{\mathrm{E}}$ with $v_{\theta'}, c_{\mathrm{D}}$.
Let $\tau = \{x_t\}_{t \in [0,1]}$ denote a trajectory, with
$p_{\theta}(\tau \mid c_{\mathrm{E}})$ and $p_{\theta'}(\tau \mid c_{\mathrm{D}})$ denoting the trajectory distributions induced by the student and teacher samplers. With a small abuse of notation, we write
$p_{\theta}(x_t,t \mid c_{\mathrm{E}})$ for the corresponding student state-time occupancy distribution obtained by sampling $\tau \sim p_\theta(\cdot \mid c_{\mathrm{E}})$ and $t \sim \mathcal{U}[0,1]$.

The goal is to train the student to solve tasks from $c_{\mathrm{E}}$, using the teacher under $c_{\mathrm{D}}$ as dense guidance.

The overall training procedure is summarized in Alg.~\ref{alg:wmsd}.

\vspace{2\baselineskip}
\paragraph{Off-policy distillation}

Matching the student velocity to the teacher velocity at teacher states gives
\begin{align}
    \mathcal{L}_{\mathrm{off}}
    =
    \mathbb{E}_{(x_t,t) \sim p_{\theta'}(\cdot \mid c_{\mathrm{D}})}
    \left[
        \left\|
        v_{\theta}(x_t,t \mid c_{\mathrm{E}})
        -
        v_{\theta'}(x_t,t \mid c_{\mathrm{D}})
        \right\|_2^2
    \right].
    \label{eq:off_policy}
\end{align}
This objective is stable because sampled states do not depend on the student~\cite{Lipman2024FlowMG}, but it constrains the student only on teacher trajectories, so errors may compound during student rollouts, a familiar issue in off-policy imitation and distillation settings~\cite{Ross2011DAgger,Agarwal2023OnPolicyDO}.

\paragraph{On-policy distillation}
To reduce this mismatch, we evaluate teacher--student discrepancy on student trajectories. Define
\begin{align}
    \ell_{\theta}(x_t,t;c_{\mathrm{E}},c_{\mathrm{D}})
    =
    \left\|
        v_{\theta}(x_t,t \mid c_{\mathrm{E}})
        -
        v_{\theta'}(x_t,t \mid c_{\mathrm{D}})
    \right\|_2^2.
    \label{eq:local_matching_loss}
\end{align}
The on-policy objective is
\begin{align}
    \mathcal{L}_{\mathrm{on}}
    =
        \mathbb{E}_{(x_t,t) \sim p_{\theta}(\cdot \mid c_{\mathrm{E}})}
    \left[
        \ell_{\theta}(x_t,t;c_{\mathrm{E}},c_{\mathrm{D}})
    \right]
    =
    \mathbb{E}_{\tau \sim p_{\theta}(\cdot \mid c_{\mathrm{E}})}
    \left[
        \int_0^1
        \ell_{\theta}(x_t,t;c_{\mathrm{E}},c_{\mathrm{D}})
        \, dt
    \right].
    \label{eq:on_policy}
\end{align}

Unlike $\mathcal{L}_{\mathrm{off}}$, $\mathcal{L}_{\mathrm{on}}$ depends on $\theta$ through both the velocity field and the student rollout distribution. Let
\begin{align}
    C_{\theta}(\tau)
    =
    \int_0^1
    \ell_{\theta}(x_t,t;c_{\mathrm{E}},c_{\mathrm{D}})
    \, dt
\end{align}
denote the trajectory-level distillation cost. For deterministic ODE sampling, $p_\theta(\tau\mid c_{\mathrm E})$ should be understood as the pushforward distribution induced by the base noise; in practice, we optimize the corresponding finite-step stochastic or flow-matching RL surrogate described in Sec.~\ref{sec:theoretical_background}. This notation motivates the score-function decomposition~\cite{Williams1992REINFORCE}
\begin{align}
    \nabla_{\theta} \mathcal{L}_{\mathrm{on}}
    =
    \mathbb{E}_{\tau \sim p_{\theta}(\cdot \mid c_{\mathrm{E}})}
    \left[
        C_{\theta}(\tau)
        \nabla_{\theta}
        \log p_{\theta}(\tau \mid c_{\mathrm{E}})
    \right]
    +
    \mathbb{E}_{\tau \sim p_{\theta}}
    \left[
        \nabla_{\theta} C_{\theta}(\tau)
    \right].
    \label{eq:on_policy_gradient_decomposition}
\end{align}
The first term changes the likelihood of trajectories according to their teacher--student discrepancy and has the form of a policy-gradient update with negative reward $-C_\theta(\tau)$. The second is direct vector-field regression on student states.
We then show that matching teacher velocity on student states, under shared initial noise, bounds student--teacher trajectory drift.
\begin{proposition}[Informal on-policy control]
\label{prop1}
Fix a paired condition $(c_{\mathrm E},c_{\mathrm D})$. Assume the teacher velocity field is $L$-Lipschitz in $x$ and that the student and teacher flows are initialized from the same base noise $x_0 \sim p_0$. If the student matches the teacher's velocity field on its own trajectories, namely
\[
    \mathbb{E}_{x_0\sim p_0}
    \left[
    \int_0^1
    \ell_\theta(x_t^\theta,t;c_{\mathrm E},c_{\mathrm D})
    \,dt
    \right]
    \leq
    \varepsilon^2,
\]
then the terminal distributions induced by the student and teacher are close. In particular, under the natural coupling given by the shared initial noise,
\begin{align}
    W_2
    \left(
        p_{\theta}(x_1 \mid c_{\mathrm{E}}),
        p_{\theta'}(x_1 \mid c_{\mathrm{D}})
    \right)
    \le 
    e^{L} \varepsilon,
\end{align}
where $\varepsilon\ge 0$.
\end{proposition}
The proof is a standard Grönwall argument (see Appendix).

\paragraph{Distillation as a reward}
Eq.~\eqref{eq:on_policy_gradient_decomposition} suggests an RL view: trajectories with low teacher--student discrepancy should become more likely. We therefore define
\begin{align}
    r_{\mathrm{distill}}(\tau)
    =
    -
    \int_0^1
    \left\|
        \operatorname{sg}\!\left[
        v_{\theta}(x_t,t \mid c_{\mathrm{E}})
        \right]
        -
        v_{\theta'}(x_t,t \mid c_{\mathrm{D}})
    \right\|_2^2
    dt,
    \label{eq:distillation_reward}
\end{align}
where $\operatorname{sg}[\cdot]$ denotes stop-gradient. Detaching the student makes this term act through trajectory likelihood rather than direct velocity-field backpropagation, up-weighting rollouts whose dynamics agree with the Demonstrator.

\paragraph{Reinforcement learning for task solving}
Pure distillation imitates the teacher but cannot systematically improve beyond it. Since eq.~\ref{eq:on_policy_gradient_decomposition} has a score-function form, we add task-level feedback. Let $
    r_{\mathrm{task}}(\tau;\mathcal{I},\mathcal{T})$
denote whether the generated video solves $\mathcal{T}$ from $\mathcal{I}$, as judged by a VLM.

The total reward is then
\begin{align}
    R(\tau)
    =
    \lambda_{\mathrm{task}}
    r_{\mathrm{task}}(\tau;\mathcal{I},\mathcal{T})
    +
    \lambda_{\mathrm{distill}}
    r_{\mathrm{distill}}(\tau),
    \label{eq:total_reward}
\end{align}
with $\lambda_{\mathrm{task}}>0$ and $\lambda_{\mathrm{distill}}>0$ controlling task success versus teacher agreement.

The teacher now acts as a stabilizing prior rather than a hard target: task reward can favor student trajectories that better satisfy the instruction even when they deviate from the teacher.

\paragraph{Optimization objective}
For the direct regression component $\mathbb{E}_{\tau \sim p_{\theta}}
    \left[
        \nabla_{\theta} C_{\theta}(\tau)
    \right]$ in eq.~\eqref{eq:on_policy_gradient_decomposition}, full backpropagation through all sampler steps is impractical. We therefore use the anchor loss
\begin{align}
    \mathcal{L}_{\mathrm{anchor}}
    =
    \mathbb{E}_{\tau \sim p_{\theta}(\cdot \mid c_{\mathrm{E}})}
    \left[
        \int_0^1
        \left\|
        v_{\theta}(\bar{x}_t,t \mid c_{\mathrm{E}})
        -
        v_{\theta'}(\bar{x}_t,t \mid c_{\mathrm{D}})
        \right\|_2^2
        dt
    \right],
    \label{eq:anchor_loss}
\end{align}

where $\bar{x}_t=\operatorname{sg}(x_t)$ is a sampled state treated as fixed. The reward term selects trajectories; the anchor keeps the student velocity close to the teacher on those states.

We optimize the student with a policy-gradient objective for flow-matching models, $\mathcal{L}_{\text{RL}}$, using eq.~\eqref{eq:total_reward} and implement it via several RL approaches~\cite{Liu2025FlowGRPO,Xue2025AdvantageWM,Zheng2025DiffusionNFTOD}: groups of rollouts for the same task define relative advantages that increase the likelihood of higher-reward rollouts, see Sec.~\ref{sec:theoretical_background}.

Finally, we combine reward optimization with teacher anchoring with $\beta_d>0$ in our \emph{full self-distillation objective}
\begin{align}
    \mathcal{L}_{\mathrm{final}}
    =
    \mathcal{L}_{\mathrm{RL}}
    +
    \beta_d
    \mathcal{L}_{\mathrm{anchor}}.
    \label{eq:final_objective}
\end{align}

Self-distillation transfers detailed execution knowledge, RL improves task success, and the Demonstrator anchor prevents uncontrolled drift while still allowing the Executor to surpass the Demonstrator under the chosen reward and evaluation protocol.

%% file: sections/04_experiments.tex
\section{Experiments}
\label{sec:experiments}

We evaluate our method along three main axes. First, we compare the proposed self-distillation variants and examine whether on-policy self-distillation provides a competitive alternative to standard off-policy self-distillation. Second, we study the interaction between self-distillation and reinforcement learning, asking whether the student can improve beyond the teacher's capabilities under VLM-based evaluation. Finally, we evaluate transfer to robotic tasks.

\subsection{Experimental Setup}
\label{sec:experimental_setup}

We operate in the Advantage-Weighted Matching (AWM) setting, a variant of GRPO better suited to flow-matching models~\cite{Xue2025AdvantageWM}. Unless otherwise stated, all experiments use a group size of 24 and a batch size of 32, with LTX-2~\cite{HaCohen2026LTX2EJ} as the baseline model. Training alternates between on-policy rollout generation, reward computation, and joint policy optimization with self-distillation. Additional implementation details are provided in Sec.~\ref{sec:appendix_implementation_details}.

\paragraph{Rewards.}

For experiments involving VLM-based reward signals, we use two complementary components: a task-completion reward and a consistency reward. Task success is evaluated with Qwen3.5-27B~\cite{qwen35blog}, which produces a binary judgment indicating whether a generated video completes the specified task. We define the reward as the log-probability difference
\[
R(x) = \log p_\text{VLM}(\text{`yes'} \mid x) - \log p_\text{VLM}(\text{`no'} \mid x),
\]
which incorporates both the predicted outcome and the model's uncertainty. However, optimizing this signal alone can lead to reward hacking, such as unrealistic object appearances or disappearances. Inspired by~\cite{Agarwal2025CosmosWF}, we introduce a consistency reward that penalizes violations of physical plausibility and temporal coherence. Full prompts and implementation details are provided in Appendix Sec.~\ref{sec:appendix_reward_prompts} and Boxes~\ref{fig:task_prompt}--\ref{fig:consistency_prompt}.

\subsubsection{WorldTasks Dataset}

We construct a dataset of 20,000 images from video-game environments and real-world scenes, largely based on MiraData~\cite{Ju2024MiraDataAL}. Standard filtering removes low-quality frames and those with limited agentic potential (i.e., no meaningful interaction possible).
For each image, we pre-generate eight task--solution pairs using Qwen3.5-27B, covering diverse instruction-following scenarios across environments and task complexities. After filtering, the resulting training split contains 146,440 task prompts. Further details are provided in Appendix Sec.~\ref{sec:appendix_further_details_dataset}.

To support learning beyond initial frames and VLM annotations, tasks are designed to be unambiguous yet general. The world model represents all visible entities, not just an egocentric view, enabling settings such as ego-exo modeling and general planning. Instructions are formatted as ``[Agent description]: [Task instruction]'' to specify the acting agent.

We further characterize the final prompt set using a prompt-only Qwen3.5-27B taxonomy. As shown in Fig.~\ref{fig:worldtasks_prompt_composition}, addressed agents are diverse but concentrated around first-person views (50.7\%) and human characters (39.0\%), with additional coverage of vehicles (5.2\%), inanimate objects or landmarks (2.4\%), creatures (1.9\%), animals (0.6\%), and crowds (0.1\%). Task types are similarly broad, dominated by positioning (22.2\%), navigation (20.3\%), object interaction (19.0\%), and perception (14.3\%). The remaining major categories include combat actions (6.6\%), compound instructions (4.6\%), vehicle control (4.1\%), UI interaction (3.7\%), and other long-tail tasks (5.1\%).

\begin{figure*}[t]
    \centering
    \begin{subfigure}[t]{0.39\textwidth}
        \centering
        \includegraphics[width=\linewidth]{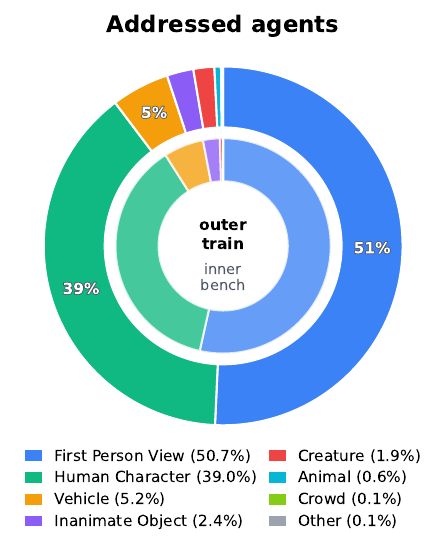}
        \caption{Addressed-agent categories.}
        \label{fig:worldtasks_agent_distribution}
    \end{subfigure}
    \hfill
    \begin{subfigure}[t]{0.59\textwidth}
        \centering
        \includegraphics[width=\linewidth]{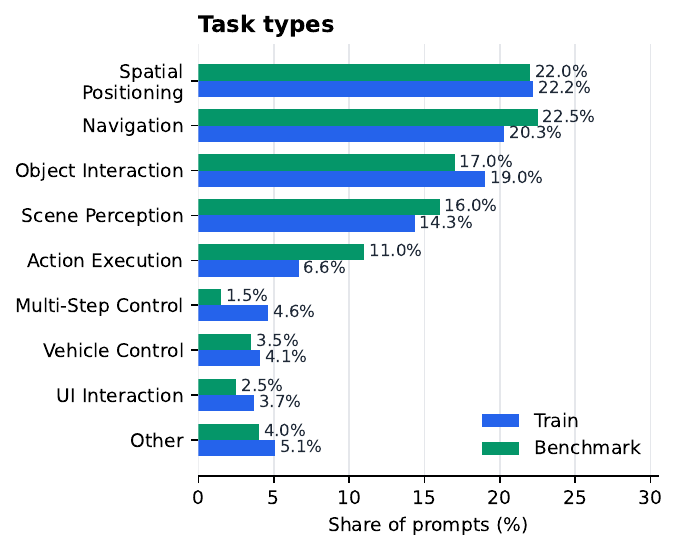}
        \caption{Task-type distribution.}
        \label{fig:worldtasks_task_distribution}
    \end{subfigure}
    \caption{WorldTasks prompt composition for the training split and WorldTasks-Bench.}
    \label{fig:worldtasks_prompt_composition}
\end{figure*}

\subsubsection{WorldTasks Benchmark}

We first study the core properties of \method in a controlled setting. The corresponding benchmark, \textbf{WorldTasks-Bench}, consists of 200 randomly selected image--task pairs. Each generated video is evaluated by a VLM according to three criteria: (1) whether the task is completed, (2) whether the correct agent attempts the task, and (3) whether the video exhibits consistent physics and realistic dynamics. The evaluation prompts are provided in Appendix Sec.~\ref{sec:evaluation_prompts}.

WorldTasks-Bench preserves the main structure of the training set: 53.5\% of benchmark prompts are first-person, 37.5\% address human characters, and 6.0\% address vehicles. The benchmark task mix is also balanced across the dominant task families: navigation tasks require moving through the environment (22.5\%), spatial positioning tasks require precise placement or alignment (22.0\%), object interaction tasks involve manipulating scene objects (17.0\%), scene perception tasks require inspecting or focusing on visual elements (16.0\%), action execution tasks involve concrete embodied actions such as combat or dodging (11.0\%), vehicle control tasks require maneuvering a vehicle or rider (3.5\%), UI interaction tasks involve menus or on-screen interface elements (2.5\%), multi-step control tasks require compound action sequences (1.5\%), and other long-tail tasks account for the remaining 4.0\%.

We report three metrics throughout all experiments. The \emph{Task Score} measures the success rate of task completion as judged by the VLM\@. The \emph{Agent Score} captures whether the intended agent engages in goal-directed interaction within the scene. The \emph{Realism Score} evaluates physical plausibility and temporal coherence. Because VLM-based evaluation can occasionally fail (e.g., due to malformed outputs or API errors), videos for which the VLM is unable to produce an assessment are discarded from the analysis; Appendix Tab.~\ref{tab:worldtasks_eval_denominators} reports the corresponding denominators and failure rates.

\subsection{On-policy vs. Off-policy Self-Distillation}

We begin by comparing the three self-distillation variants introduced in Sec.~\ref{sec:method}: off-policy self-distillation, on-policy self-distillation using only the anchor loss between student and teacher, and the full on-policy self-distillation objective in Eq.~\eqref{eq:final_objective}. In Fig.~\ref{fig:ablation_results}, we report evaluation results every 10 training steps over 100 training steps. We show both the average WorldTasks score and PickScore~\cite{Kirstain2023PickaPicAO}, which measures overall generation quality.

All three methods yield substantial improvements. However, after approximately 60 training steps, off-policy self-distillation saturates, whereas both on-policy variants continue to improve on both metrics and ultimately surpass the off-policy baseline. The full on-policy self-distillation objective, which includes the distillation reward introduced in Eq.~\ref{eq:distillation_reward}, achieves the best overall performance.

\begin{figure}[t]
    \centering
    \begin{subfigure}[t]{0.60\linewidth}
        \centering
        \includegraphics[width=\linewidth]{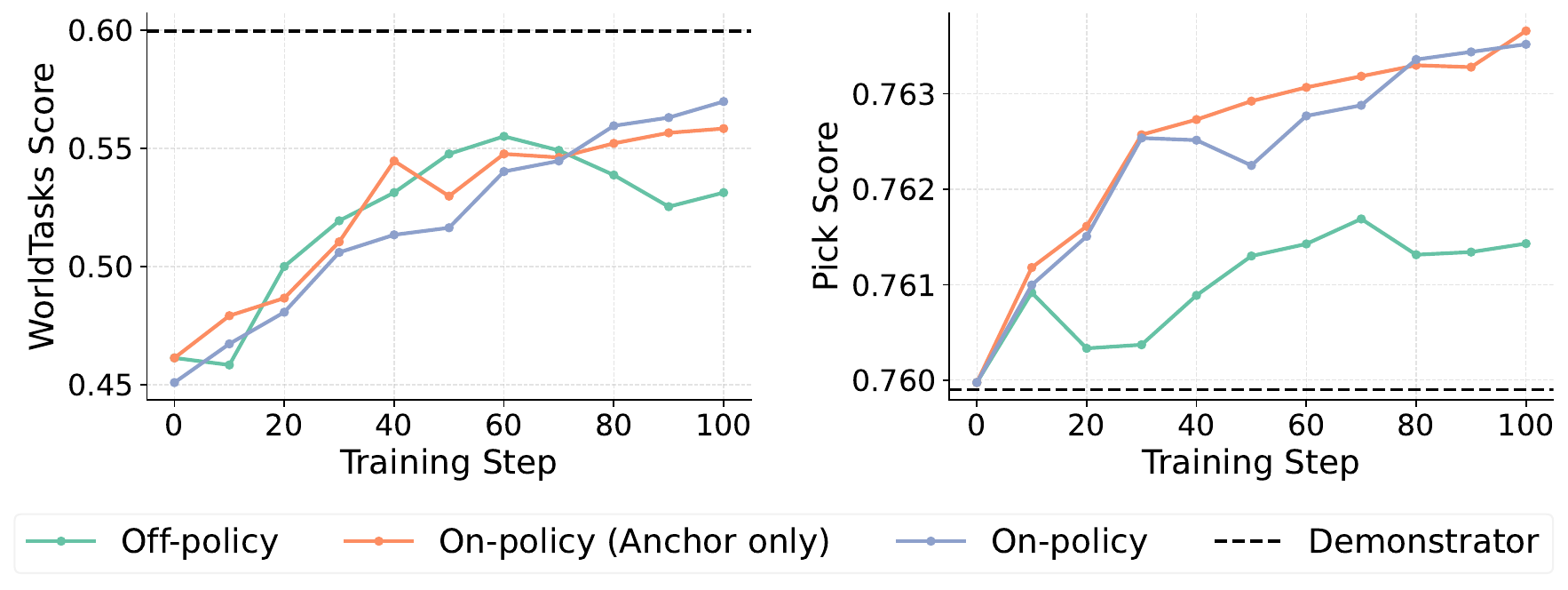}
        \caption{}
    \end{subfigure}
    \hfill
    \begin{subfigure}[t]{0.38\linewidth}
        \centering
        \includegraphics[width=\linewidth]{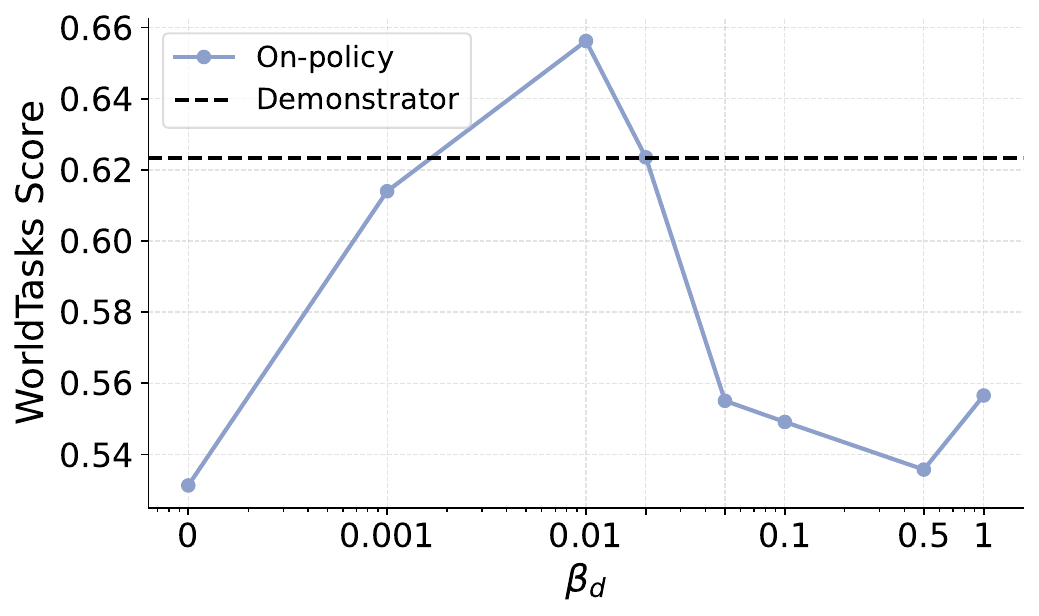}
        \caption{}
    \end{subfigure}
    \caption{
    \textbf{Two ablations on \benchname.}
    Left: Ablation on self-distillation methods, showing average WorldTasks score and PickScore.
    Right: Ablation of average WorldTasks score vs.\ $\beta_d$.
    }
    \label{fig:ablation_results}
\end{figure}

\begin{figure}[t]
    \centering
    \includegraphics[width=0.98\linewidth]{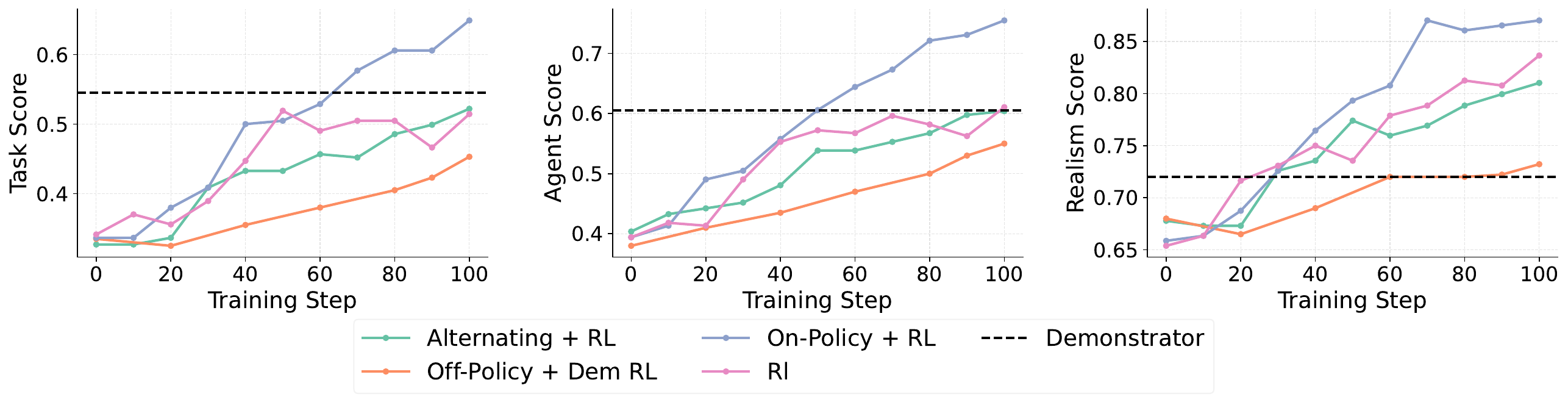}
        \caption{
    \textbf{Ablation across training settings on \benchname.}
    We report the three evaluation dimensions.
    }
    \label{fig:rl_ablation}
\end{figure}

\subsection{Surpassing the Demonstrator with RL Training}

We investigate whether augmenting self-distillation with reinforcement learning (RL) enables the student to surpass the demonstrator's task-solving performance. To this end, we consider four training settings: (i) standard RL without an anchor loss, (ii) on-policy self-distillation with RL applied to the student, (iii) off-policy distillation with RL applied to the teacher, and (iv) an alternating optimization strategy in which teacher and student updates are interleaved according to a fixed schedule. The full procedure is detailed in Alg.~\ref{alg:alt_grpo_demo_simple} (see Appendix). As an additional baseline, we include the \emph{Demonstrator} setting, in which reasoning is entirely outsourced to the VLM for solution generation.

We evaluate all approaches on the three components of \benchname: task-solving performance, agent correctness, and physical consistency. The results are shown in Fig.~\ref{fig:rl_ablation}.

Our results show that combining on-policy self-distillation with RL substantially improves task-solving performance and enables the student to surpass the demonstrator under the VLM-based benchmark. In contrast, standard RL alone achieves comparable performance during early training, up to approximately 50 steps, but quickly plateaus and yields no further gains. The remaining approaches exhibit slower learning dynamics and do not reach the same level of performance.

\subsection{Comparison to Baselines}
\label{sec:comparison_to_baselines}

We compare \method against several baselines. We first examine whether \method generalizes across different base models, reward functions, and RL optimization settings. To this end, we use HunyuanVideo-1.5~\cite{hunyuanvideo2025} as the base model, Qwen3-VL-8B as the reward model, and FlowGRPO~\cite{Liu2025FlowGRPO} as the RL optimizer, training for 25 steps.

For the LTX-2~\cite{HaCohen2026LTX2EJ} 8-step model, using the setup described in Sec.~\ref{sec:experimental_setup}, we compare against multiple baselines. First, we consider direct solution generation by outsourcing reasoning to a VLM\@. In this setting, the first frame and task description are provided to the VLM, which generates an image-to-video solution prompt that is then used for video generation (+\textbf{VLM}). We also investigate whether unannotated videos can be converted into task--video pairs by labeling them with corresponding tasks and subsequently fine-tuning the model via supervised fine-tuning (+\textbf{SFT}). Finally, we compare against \method. All results on \benchname are reported in Tab.~\ref{tab:world_tasks_bench_results}.

After only 25 training steps, applying \method to the HY1.5 baseline improves all reported metrics, highlighting its robustness across training settings. Applying \method to the 8-step distilled LTX-2 model yields larger gains in task completion and agent correctness while also improving physical consistency. In contrast, the SFT baseline provides little to no improvement and, in some cases, degrades performance. We hypothesize that this is due to limited task diversity in the automatically annotated data: many tasks are overly simple or repetitive, such as ``walk forward'', and therefore fail to capture meaningful real-world interaction scenarios.

\begin{table}[t]
\caption{\textbf{Comparison of \method against baselines on \benchname.} We report task completion, agent correctness, physical consistency, their average, and end-to-end inference time. \method consistently improves task and agent scores across both base models while preserving the inference cost of the underlying model. Dark blue indicates improvements obtained with \method over the corresponding base model, while bold black denotes the best overall result. $^{*}$ Trained with GRPO for 25 steps.}
\label{tab:world_tasks_bench_results}
\centering
\small
\setlength{\tabcolsep}{4pt}
\begin{tabularx}{\textwidth}{
>{\raggedright\arraybackslash}p{0.31\textwidth}
>{\centering\arraybackslash}X
>{\centering\arraybackslash}X
>{\centering\arraybackslash}X
>{\centering\arraybackslash}X
>{\centering\arraybackslash}X
}
\toprule
\textbf{Method} &
\textbf{Task} $\uparrow$ &
\textbf{Agent} $\uparrow$ &
\textbf{Consist.} $\uparrow$ &
\textbf{Avg.} $\uparrow$ &
\textbf{E2E (s)} $\downarrow$ \\
\midrule
HY1.5 & 0.464 & 0.540 & 0.780 & 0.597 & 112 \\
\rowcolor{gray!10}
\textbf{\textcolor{wmsdblue}{HY1.5+\textit{WMSD}$^{*}$}} &
\textbf{\textcolor{wmsdblue}{0.574}} &
\textbf{\textcolor{wmsdblue}{0.630}} &
\textbf{\textcolor{wmsdblue}{0.828}} &
\textbf{\textcolor{wmsdblue}{0.673}} &
112 \\
\midrule
LTX-2 & 0.315 & 0.395 & 0.690 & 0.467 & 52.2 \\
LTX-2+SFT & 0.292 & 0.389 & 0.682 & 0.454 & 52.2 \\
\rowcolor{gray!10}
\textbf{\textcolor{wmsdblue}{LTX-2+\textit{WMSD}$^{*}$}} &
\textbf{\textcolor{wmsdblue}{0.452}} &
\textbf{\textcolor{wmsdblue}{0.500}} &
\textbf{\textcolor{wmsdblue}{0.693}} &
\textbf{\textcolor{wmsdblue}{0.548}} &
52.2 \\
\midrule
LTX-2 (8-Step) & 0.285 & 0.391 & 0.694 & 0.455 & 10.1 \\
LTX-2 (8-Step)+VLM & 0.495 & 0.572 & 0.732 & 0.598 & 10.5 \\
\rowcolor{gray!10}
\textbf{LTX-2 (8-Step)+\method} &
\textbf{0.605} &
\textbf{0.691} &
\textbf{0.882} &
\textbf{0.726} &
\textbf{10.1} \\
\bottomrule
\end{tabularx}
\end{table}

The quantitative improvements are further reflected in qualitative comparisons shown in Fig.~\ref{fig:qualitative_mixed_examples}. Across both LTX-2 and HunyuanVideo-1.5, videos generated with \method exhibit more accurate task execution, stronger agent--environment interaction, and improved physical consistency compared to the corresponding base models. In particular, \method produces trajectories that better align with the intended task objectives while maintaining coherent motion and scene dynamics over time.
\begin{figure}[t]
    \centering
    \begin{subfigure}[t]{0.95\linewidth}
        \centering
        \includegraphics[width=\linewidth]{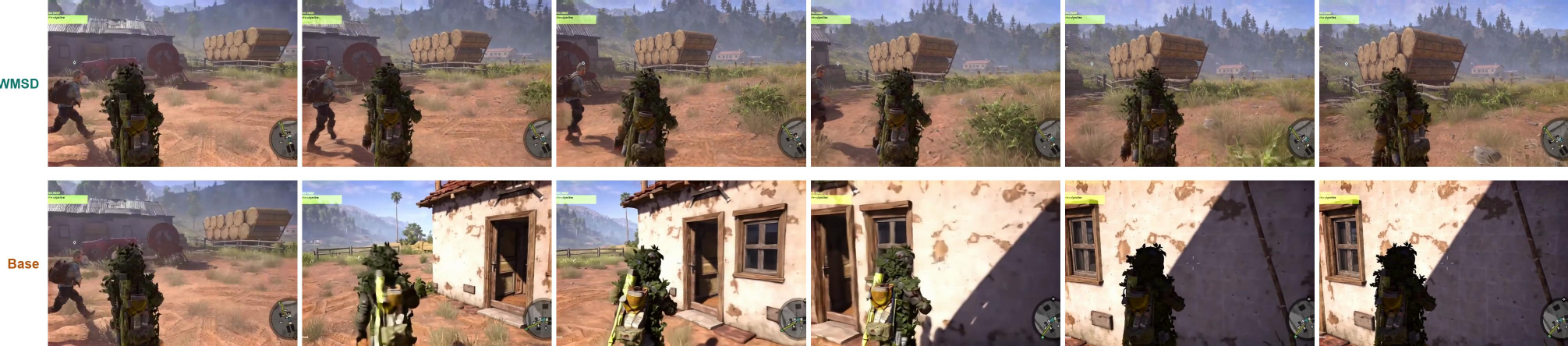}
        \caption{\textit{[First-person view]: Move to the right to examine the distant house.}}
    \end{subfigure}
    \vspace{0.1cm}
    \begin{subfigure}[t]{0.95\linewidth}
        \centering
        \includegraphics[width=\linewidth]{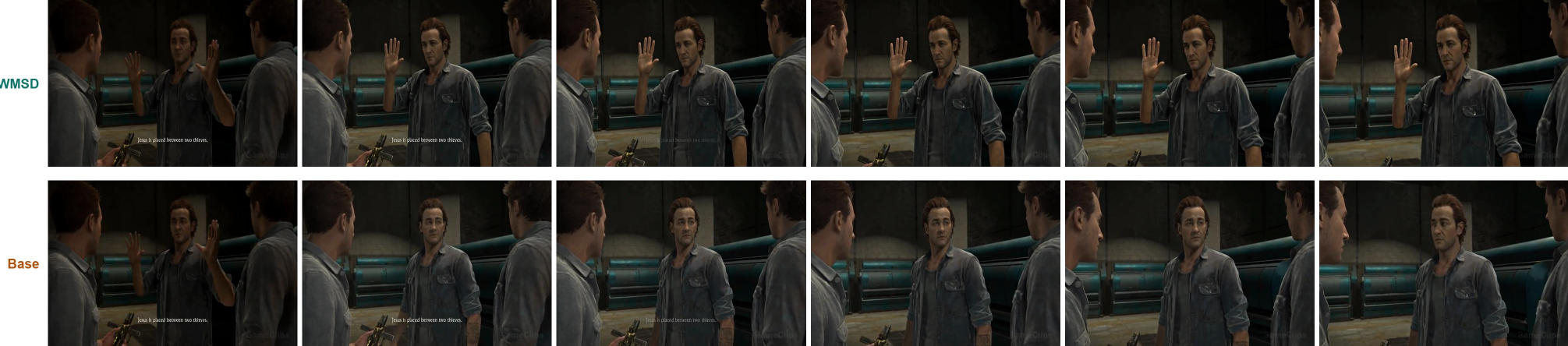}
        \caption{\textit{[Man in center]: Lower one hand slowly.}}
    \end{subfigure}
    \vspace{0.1cm}
    \begin{subfigure}[t]{0.95\linewidth}
        \centering
        \includegraphics[width=\linewidth]{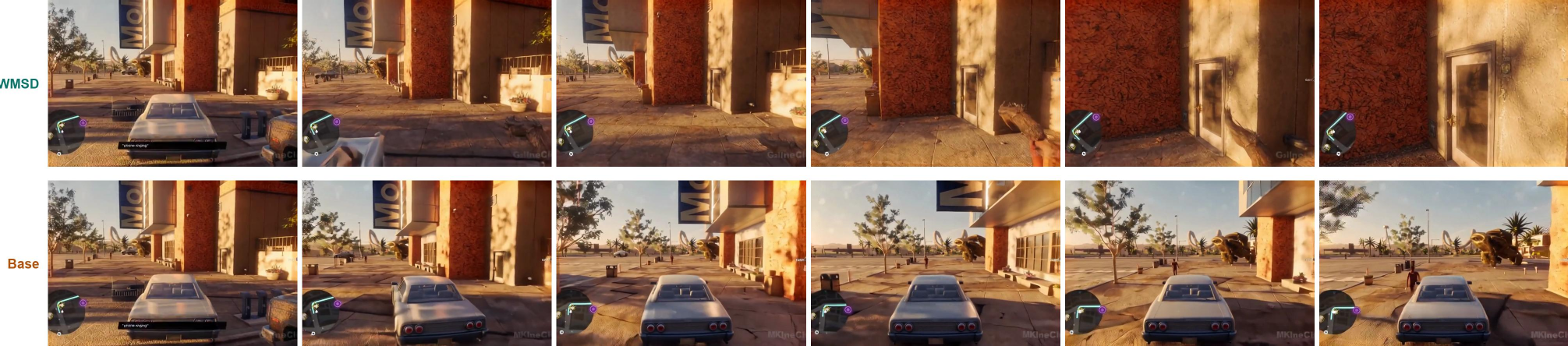}
        \caption{\textit{[First-person view]: Walk toward the building entrance on the right.}}
    \end{subfigure}
    \vspace{0.1cm}
    \begin{subfigure}[t]{0.95\linewidth}
        \centering
        \includegraphics[width=\linewidth]{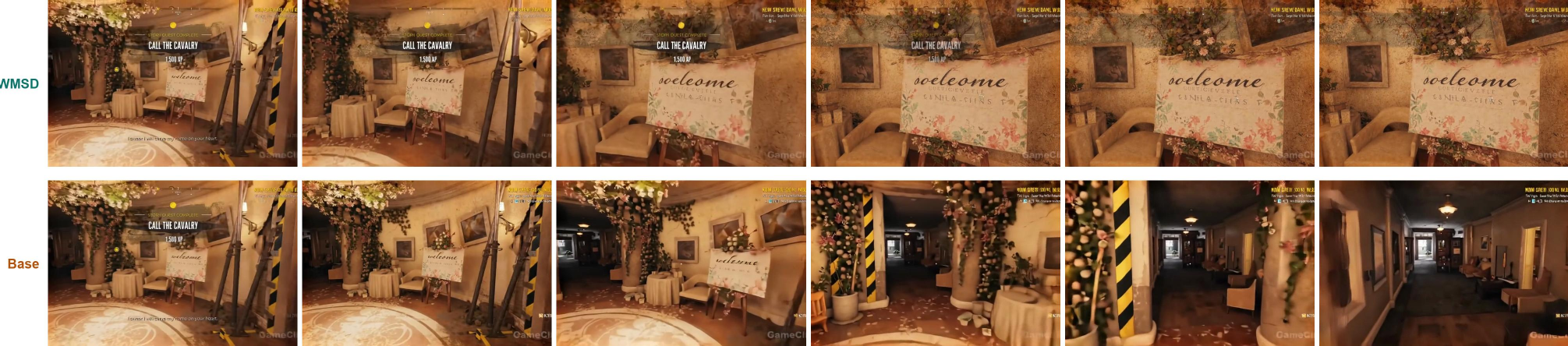}
        \caption{\textit{[First-person view]: Turn right to face the welcome sign on the easel.}}
    \end{subfigure}
    \vspace{0.1cm}
    \begin{subfigure}[t]{0.95\linewidth}
        \centering
        \includegraphics[width=\linewidth]{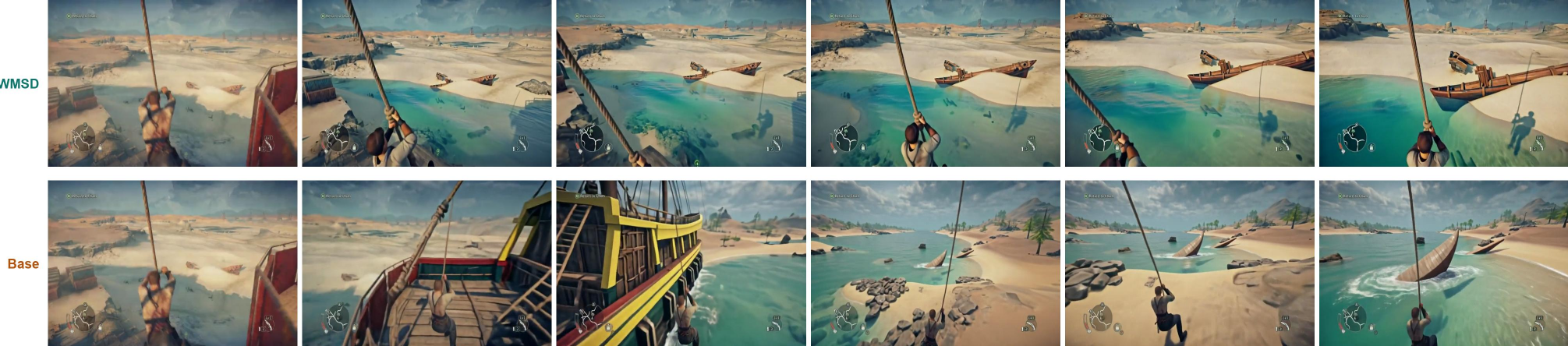}
        \caption{\textit{[First-person view]: Use the rope to swing toward the nearest visible shipwreck's hull.}}
    \end{subfigure}

    \caption{Qualitative comparisons between \method and the base model across LTX-2 and HunyuanVideo-1.5. Each subfigure shows six uniformly sampled frames from the generated videos.}
    \label{fig:qualitative_mixed_examples}
\end{figure}

\subsection{Performance Breakdown}
\label{sec:performance_breakdown}
We investigate which task and agent categories \methodname handles best. We stratify \benchname by the VLM-derived prompt taxonomy introduced in Sec.~\ref{sec:experimental_setup} and analyze all categories that cover more than 5\% of the benchmark. This yields five task types (navigation, positioning, object interaction, perception, combat action) and three addressed-agent types (first-person view, human character, vehicle). Fig.~\ref{fig:worldtasks_performance_breakdown} reports Task Score across task types and Agent Score across addressed-agent types.

\begin{figure}[t]
    \centering
    \includegraphics[width=\linewidth]{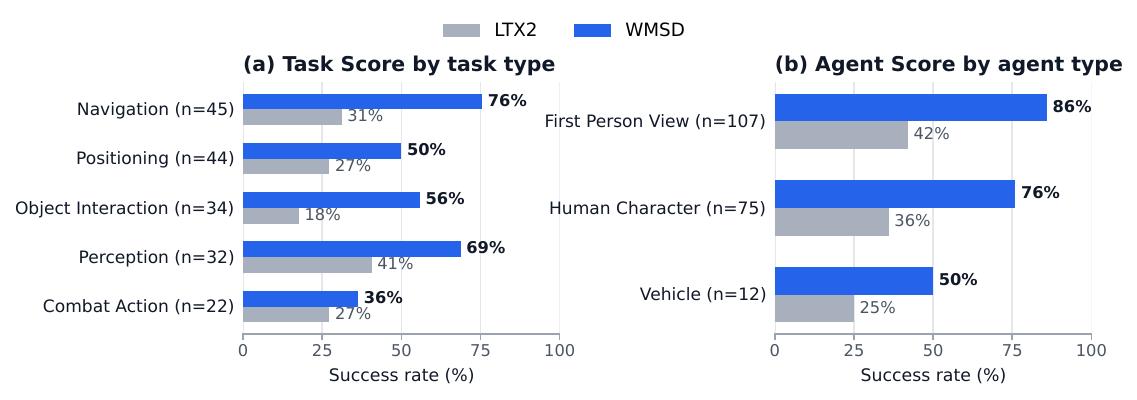}
    \caption{
    \textbf{Performance breakdown on \benchname.}
    Left: Task Score by task type. Right: Agent Score by addressed-agent type. We show all categories with more than 5\% benchmark support. Values are VLM-judged success rates in percent, with subgroup sizes shown in parentheses.
    }
    \label{fig:worldtasks_performance_breakdown}
\end{figure}

The breakdown shows that \methodname improves task completion most strongly for navigation (31.1\% to 75.6\%) and object interaction (17.6\% to 55.9\%), while also improving perception (40.6\% to 68.8\%), positioning (27.3\% to 50.0\%), and combat actions (27.3\% to 36.4\%). Agent grounding improves substantially for first-person prompts (42.1\% to 86.0\%) and human-character prompts (36.0\% to 76.0\%). Vehicle prompts remain more challenging, reaching 50.0\% Agent Score, but this slice contains only 12 examples and therefore should be interpreted as a small-support diagnostic rather than a primary trend.

\subsection{Ablation Studies}
\label{sec:ablation_studies}

\paragraph{Self-distillation strength.}
We study the effect of the self-distillation strength on the performance of \methodname by varying the Demonstrator anchor coefficient, $\beta_d$ in Eq.~\ref{eq:final_objective}, over the range $[0,1]$. As shown in Fig.~\ref{fig:ablation_results}, the best performance is obtained around $\beta_d = 0.01$. Both smaller and larger values perform worse: too little regularization weakens the distillation signal, whereas too much regularization dominates the RL objective and limits learning.

\paragraph{Consistency reward.}
We further investigate the effect of the additional consistency reward, which is designed to mitigate reward hacking. Without this reward, the model can exploit the VLM reward by producing implausible generations, such as objects appearing or disappearing without physical justification. The exact prompt used for this reward is provided in Box~\ref{fig:consistency_prompt}. Fig.~\ref{fig:consistency_pair} shows qualitative examples with and without the consistency reward.

\paragraph{Resolution and inference steps.}
Following prior work~\cite{Liu2025FlowGRPO, He2025TempFlowGRPOWT}, we decouple the number of denoising steps and the resolution used during training and evaluation to improve training efficiency. However, we find that this introduces a trade-off: lower generation quality during training increases the risk of reward hacking, especially in our setting where the VLM requires clear and unambiguous visual evidence to assign reliable rewards.

Additional ablations are provided in the Appendix.
\begin{figure}[t]
    \centering
    
    \begin{subfigure}{0.95\textwidth}
        \centering
        \includegraphics[width=\linewidth]{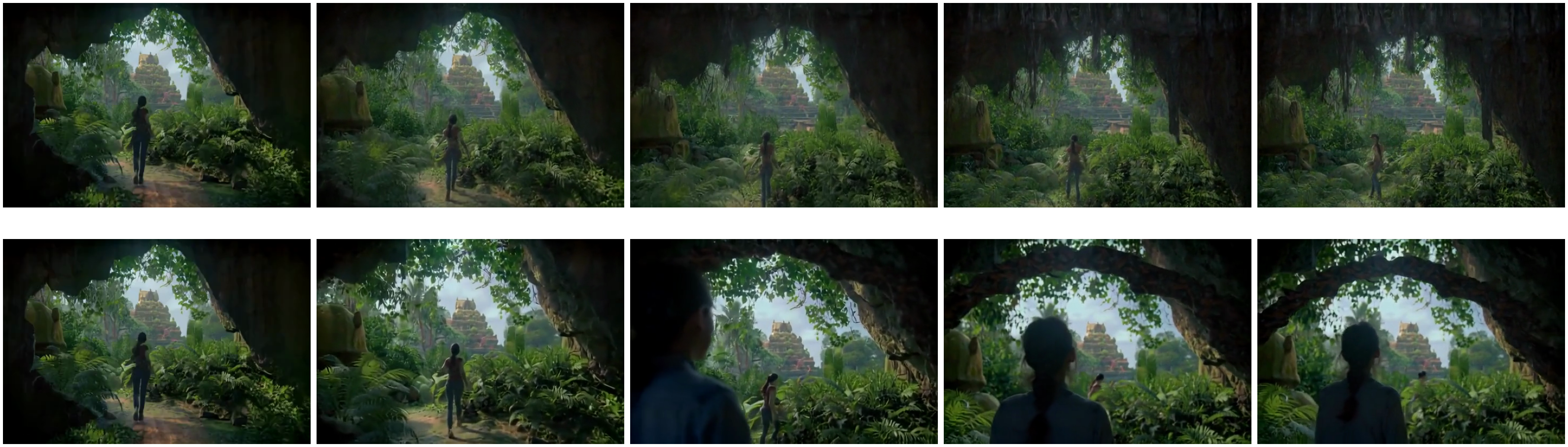}
        \caption{Prompt: ``\textit{[Lara Croft]: Turn left and position yourself directly in front of the central arch.}''}
        \label{fig:consistency_5}
    \end{subfigure}
    
    \vspace{0.5cm}
    
    \begin{subfigure}{0.95\textwidth}
        \centering
        \includegraphics[width=\linewidth]{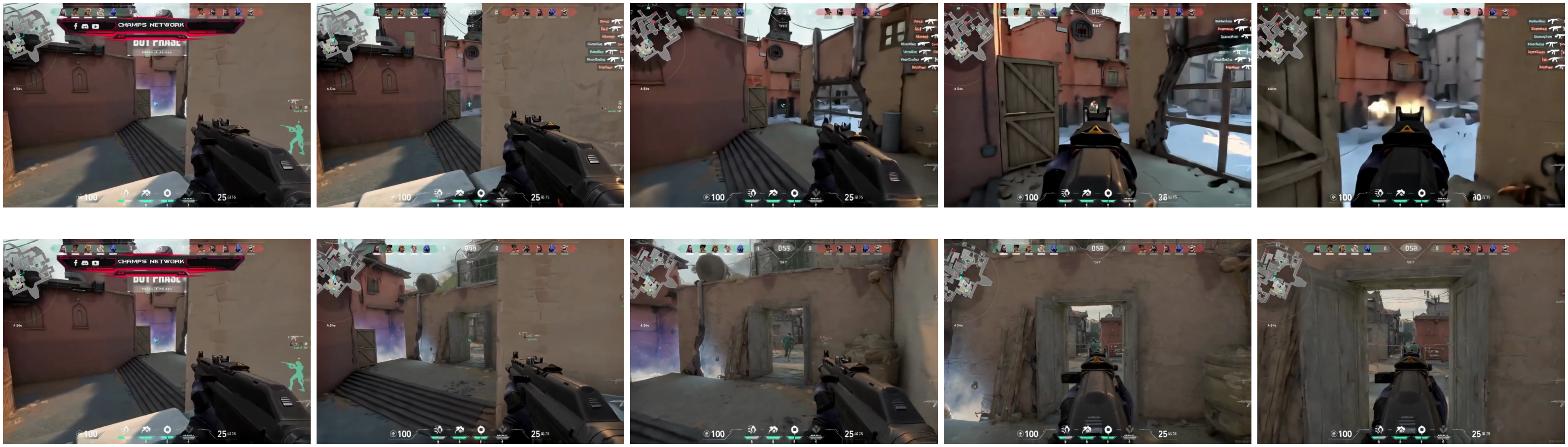}
       \caption{Prompt: ``\textit{[First-person view]: Aim the weapon at the doorway entrance.}''}
       \label{fig:consistency_6}
    \end{subfigure}
    
    \caption{Two examples: the first row uses the consistency reward, while the second row does not. The second row shows that the model generates (a) the arch + Lara Croft and (b) the doorway as a consequence of reward hacking.}
    \label{fig:consistency_pair}
\end{figure}

\subsection{Generalization to Robotic Tasks}
\label{sec:robotic_tasks}

An important application of world models for planning lies in robotics, where data collection is particularly expensive. We therefore evaluate whether \method trained on \datasetname can achieve competitive performance without task-specific supervision, compared to supervised fine-tuning (SFT) on the Gr00t dataset using the DreamGen benchmark~\cite{Jang2025DreamGenUG} (Tab.~\ref{tab:dreamgenbench}).

We compare our LTX-2-based model against several baselines, including HunyuanVideo (Huny)~\cite{Kong2024HunyuanVideoAS}, CogVideoX (CogX)~\cite{Hong2022CogVideoLP}, Wan~\cite{Wang2025WanOA}, and Cosmos~\cite{Agarwal2025CosmosWF}, across zero-shot and SFT settings.

We observe that, despite operating in a data-free regime, \method achieves performance comparable to SFT-trained Cosmos, while substantially improving over the LTX-2 baseline. 

%
%

\begin{table*}[t]
\caption{\textbf{Performance on the DreamGen benchmark.} We compare zero-shot and SFT baselines against \method. Despite not using task-specific supervision, \method achieves competitive performance with SFT-trained models. Best results are in bold, second-best underlined.}
\label{tab:dreamgenbench}
\centering
\small
\setlength{\tabcolsep}{2.8pt}
\renewcommand{\arraystretch}{1.12}
\begin{tabularx}{\textwidth}{@{}>{\raggedright\arraybackslash}p{0.12\textwidth} ZZZZZ SSSS D@{}}
\toprule
& \multicolumn{5}{c}{\cellcolor{zeroshotbg}\textbf{Zero-shot}}
& \multicolumn{4}{c}{\cellcolor{sftbg}\textbf{SFT}}
& \multicolumn{1}{c@{}}{\cellcolor{wsdmbg}\textbf{\method}} \\
\cmidrule(lr){2-6} \cmidrule(lr){7-10} \cmidrule(l){11-11}
\textbf{Metric}
& Huny. & CogX & Wan & Cosmos & LTX-2
& Huny. & CogX & Wan & Cosmos
& LTX-2 \\
\midrule
Object
& 0.0 & 0.0 & 2.0 & 32.0 & 20.0
& 26.0 & 38.0 & 58.0 & \underline{62.0}
& \textbf{70.0} \\

Behavior
& 2.1 & 0.0 & 2.1 & 31.9 & 29.8
& 10.6 & 28.0 & 55.3 & \textbf{61.7}
& \underline{57.4} \\

Environment
& 0.0 & 0.0 & 6.7 & 24.1 & 41.4
& 27.6 & 41.4 & \textbf{65.5} & \textbf{65.5}
& \underline{58.6} \\
\bottomrule
\end{tabularx}
\end{table*}

\begin{figure}[t]
    \centering
    \includegraphics[width=0.98\textwidth]{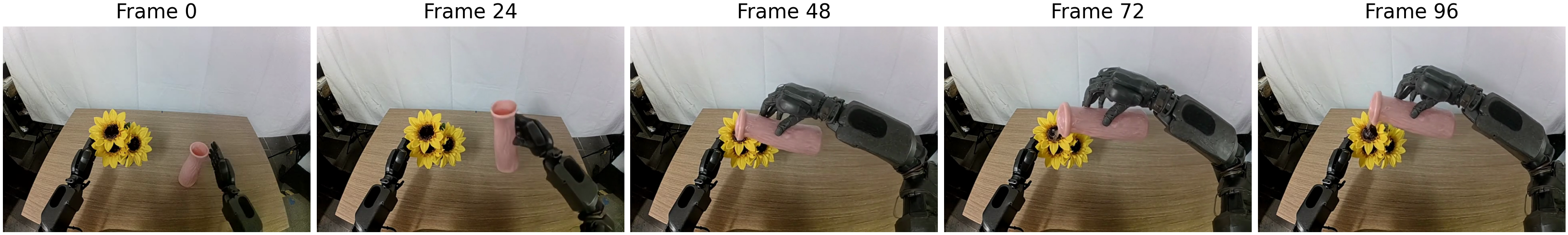}
    \caption{Example video generated with \method and LTX-2 on the DreamGen benchmark. Task: \textit{Use the right hand to pick up the pink bottle and pour water on the flower.}}
    \label{fig:gr00t_example}
\end{figure}

\subsection{Discussion \& Limitations}

\paragraph{Generalizability}
Training with \method leads to substantial improvements on \benchname as well as on robotics-related tasks (Sec.~\ref{sec:robotic_tasks}), achieving performance competitive with supervised fine-tuning. Furthermore, recent advances in distilling video generators into few-step models enable efficient RL-based optimization. We show that \method can effectively leverage these distilled models, resulting in significant gains in training efficiency.
\paragraph{Limits of Data-Free Training}
The results on the DreamGen benchmark highlight an inherent limitation of the data-free setting. In particular, the model cannot recover accurate robot-specific dynamics without access to corresponding data. While \method generates plausible task solutions, it lacks detailed knowledge of the appearance and motion characteristics of a specific robotic platform beyond the initial frame; see Fig.~\ref{fig:gr00t_example}. This limitation is intrinsic to the data-free nature of \method. Extending \method to video continuation and in-context learning (ICL) could resolve this issue.
\paragraph{Out-of-Distribution Tasks}
We further investigate performance on out-of-distribution tasks, such as puzzle-based games~\cite{Wang2026AVB}. As detailed in Appendix Sec.~\ref{sec:appendix_vbvr}, when the Demonstrator fails to produce coherent solutions, \method still yields improvements, albeit with diminished gains. This observation motivated the alternating RL training strategy (Sec.~\ref{alg:alt_grpo_demo_simple}); however, as shown in Fig.~\ref{fig:rl_ablation}, this approach introduces additional instability.

%% file: sections/appendix.tex
\section{Technical appendices and supplementary material}
\subsection{Further Implementation Details}
\label{sec:appendix_implementation_details}

In Tab.~\ref{tab:training_hparams_comparison}, we report the hyperparameters used for self-distilling LTX-2 and HunyuanVideo-1.5 in the experiments from Sec.~\ref{sec:experiments}. 
\begin{table*}[t]
\centering
\small
\begin{tabular}{lll}
\toprule
\textbf{Category} & \textbf{LTX-2} & \textbf{HunyuanVideo-1.5} \\
\midrule
Base model & \texttt{LTX-2-19b-distilled} & \texttt{HunyuanVideo-1.5-480p\_i2v} \\
Model type & \texttt{ltx2\_i2v} & \texttt{hy15\_i2v} \\
Fine-tuning method & LoRA & LoRA \\
LoRA rank & 64 & 64 \\
LoRA alpha & 128 & 128 \\
Trainer type & \texttt{AWM-demo} & \texttt{GRPO-demo} \\
Loss type & \texttt{exp\_first} & -- \\
Advantage weighting & \texttt{ghuber} & -- \\
Advantage aggregation & \texttt{gdpo} & -- \\
Learning rate & $2 \times 10^{-4}$ & $1 \times 10^{-4}$ \\
Optimizer & Adam & Adam \\
Adam betas & $(0.9,\,0.999)$ & $(0.9,\,0.999)$ \\
Adam epsilon & $10^{-8}$ & $10^{-8}$ \\
Weight decay & $10^{-4}$ & $10^{-4}$ \\
Max grad norm & 1.0 & 1.0 \\
EMA decay & 0.96 & 0.9 \\
EMA decay schedule & constant & power \\
EMA update interval & 1 & 4 \\
Number of inner epochs & 1 & 1 \\
Unique samples per epoch & 32 & 8 \\
Group size & 24 & 16 \\
Training resolution & $384 \times 576$ & $480 \times 848$ \\
Evaluation resolution & $512 \times 768$ & $480 \times 848$ \\
Number of frames & 121 & 121 \\
Inference steps (train) & 8 & 10 \\
Inference steps (eval) & 8 & 40 \\
Training timesteps & 4 & 10 \\
Timestep range & $[0,\,1]$ & $[0,\,0.9]$ \\
Clip range & $[-1, 1]$ & $[-3\!\times\!10^{-4},\,3\!\times\!10^{-4}]$ \\
Advantage clip range & $[-5, 5]$ & $[-5, 5]$ \\
$\beta_d$ & 0.008 & 1.0 \\

\midrule
\multicolumn{3}{c}{\textit{Reward}} \\
\midrule
Reward model & \texttt{Qwen3.5-27B-FP8} & \texttt{Qwen3-VL-8B-Instruct} \\
Reward type & multi (task/consistency/pick) & task+consistency \\
Reward weight & 0.5 / 0.225 / 0.225 & 0.95 \\
Distillation reward weight & 0.05 & 0.05 \\
\bottomrule

\end{tabular}
\caption{Main training hyperparameters used for fine-tuning LTX-2 and HunyuanVideo-1.5. For details on implementation-specific parameters, we refer to the official codebase.}
\label{tab:training_hparams_comparison}
\end{table*}

\subsection{Compute Resources}
The primary results presented in Sec.~\ref{sec:comparison_to_baselines} were obtained using a large-scale cluster comprising 128 GH200 GPUs.  
In contrast, the ablation studies in Sec.~\ref{sec:ablation_studies} were conducted on a smaller setup of 16 GH200 GPUs over a 12-hour period.

\subsubsection{Distribution-Matching Self-Distillation}
\label{sec:dmd_self_distillation}

We also investigated Distribution Matching Distillation (DMD) as an alternative on-policy distillation objective combined with RL, following~\cite{Jiang2025DistributionMD}.

As background, diffusion models generate high-quality samples by iteratively denoising Gaussian noise. However, this multi-step sampling process is computationally expensive, motivating distillation into a one-step generator \(G_\theta(z)\). Distribution Matching Distillation trains \(G_\theta\) to match the distribution of a pretrained teacher diffusion model, rather than exactly reproducing its full denoising trajectory. The core objective is to align the generator-induced distribution with the teacher distribution by minimizing the KL divergence:
\begin{align}
D_{\mathrm{KL}}(p_{\mathrm{fake}} \,\|\, p_{\mathrm{real}})
&=
\mathbb{E}_{x \sim p_{\mathrm{fake}}}
\left[
\log \frac{p_{\mathrm{fake}}(x)}{p_{\mathrm{real}}(x)}
\right].
\end{align}
The corresponding score-based gradient is given by
\begin{align}
\nabla_\theta D_{\mathrm{KL}}
&=
\mathbb{E}_{z}
\left[
-
\left(
s_{\mathrm{real}}(x)-s_{\mathrm{fake}}(x)
\right)
\frac{\partial G_\theta(z)}{\partial \theta}
\right],
\qquad x = G_\theta(z).
\end{align}

Because diffusion models estimate scores on noisy samples, DMD perturbs generated samples according to
\begin{align}
q_t(x_t \mid x)
&=
\mathcal{N}(\alpha_t x,\sigma_t^2 I),
\end{align}
and estimates the real and fake scores using diffusion denoisers:
\begin{align}
s_{\mathrm{real}}(x_t,t)
&=
-\frac{x_t-\alpha_t \mu_{\mathrm{base}}(x_t,t)}{\sigma_t^2}, \\
s_{\mathrm{fake}}(x_t,t)
&=
-\frac{x_t-\alpha_t \mu^\phi_{\mathrm{fake}}(x_t,t)}{\sigma_t^2}.
\end{align}
The fake-score denoiser is trained online with the denoising objective
\begin{align}
\mathcal{L}^{\phi}_{\mathrm{denoise}}
&=
\left\|
\mu^\phi_{\mathrm{fake}}(x_t,t)-x
\right\|_2^2,
\end{align}
while the generator is updated using the approximate distribution-matching gradient
\begin{align}
\nabla_\theta D_{\mathrm{KL}}
&\simeq
\mathbb{E}_{z,t,x,x_t}
\left[
w_t \alpha_t
\left(
s_{\mathrm{fake}}(x_t,t)
-
s_{\mathrm{real}}(x_t,t)
\right)
\frac{\partial G_\theta(z)}{\partial \theta}
\right].
\end{align}

We adapt this objective to our demonstrator--executor self-distillation setting by minimizing the KL divergence between the executor distribution \(p_{\theta}(x_t,t \mid c_{\mathrm{E}})\) and the demonstrator distribution \(p_{\theta'}(x_t,t \mid c_{\mathrm{D}})\):
\begin{align}
D_{\mathrm{KL}}
\left(
p_{\theta}(x_t,t \mid c_{\mathrm{E}})
\,\|\, 
p_{\theta'}(x_t,t \mid c_{\mathrm{D}})
\right).
\end{align}
Taking the gradient with respect to the executor parameters \(\theta\) yields the approximation
\begin{align}
&\nabla_\theta
D_{\mathrm{KL}}\!\left(
p_{\theta}(x_t,t \mid c_{\mathrm{E}})
\,\|\, 
p_{\theta'}(x_t,t \mid c_{\mathrm{D}})
\right)
\simeq \notag\\
&\quad
\mathbb{E}_{z,t,x,x_t}
\Biggl[
w_t \alpha_t
\left(
s_{\theta}(x_t,t \mid c_{\mathrm{E}})
-
s_{\theta'}(x_t,t \mid c_{\mathrm{D}})
\right) \notag\\
&\hspace{6.5em}
\times
\frac{\partial G_\theta(z,c_{\mathrm{E}})}{\partial \theta}
\Biggr],
\\
&\quad
x = G_\theta(z,c_{\mathrm{E}}),
\qquad
x_t \sim q_t(x_t \mid x).
\end{align}

Although this objective is conceptually appealing, we found it unstable in practice. Across our experiments, the DMD-based self-distillation objective consistently diverged, and we therefore did not use it in the final method.

\subsubsection{Reward Prompts}
\label{sec:appendix_reward_prompts}
We present the reward prompts used during training, shown in Fig.~\ref{fig:task_prompt} and Fig.~\ref{fig:consistency_prompt}. These prompts provide binary supervision for task success and visual-temporal consistency, enabling stable reward computation from generated videos.

\begin{figure}[t]
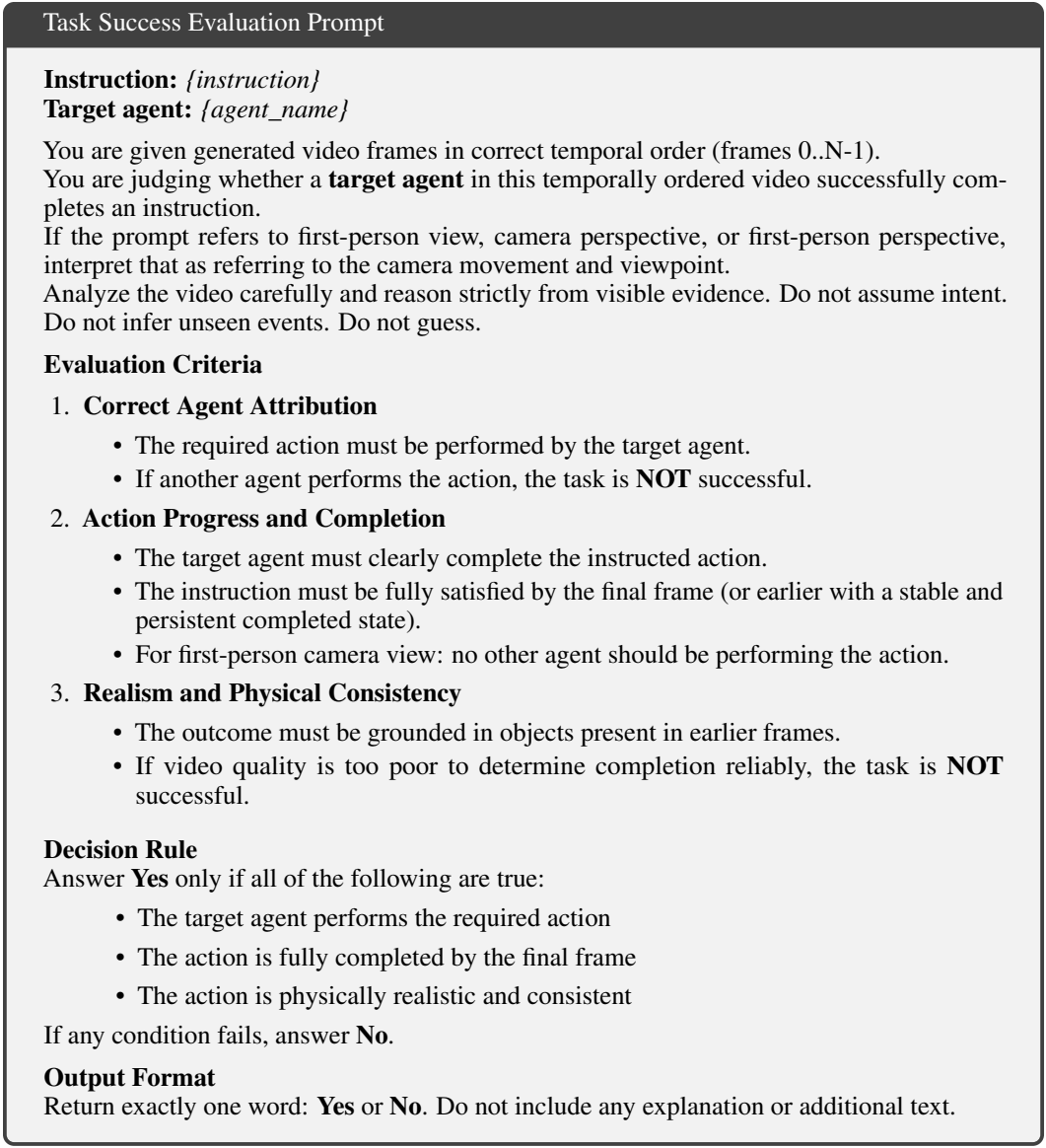

\centering
\begin{tcolorbox}[title=Task Success Evaluation Prompt]

\textbf{Instruction:} \textit{\{instruction\}} \\
\textbf{Target agent:} \textit{\{agent\_name\}}

\vspace{0.5em}

You are given generated video frames in correct temporal order (frames 0..N-1).

You are judging whether a \textbf{target agent} in this temporally ordered video successfully completes an instruction.

If the prompt refers to first-person view, camera perspective, or first-person perspective, interpret that as referring to the camera movement and viewpoint.

Analyze the video carefully and reason strictly from visible evidence.  
Do not assume intent. Do not infer unseen events. Do not guess.

\vspace{0.5em}
\textbf{Evaluation Criteria}

\begin{enumerate}[leftmargin=1.5em]

\item \textbf{Correct Agent Attribution}
\begin{itemize}
    \item The required action must be performed by the target agent.
    \item If another agent performs the action, the task is \textbf{NOT} successful.
\end{itemize}

\item \textbf{Action Progress and Completion}
\begin{itemize}
    \item The target agent must clearly complete the instructed action.
    \item The instruction must be fully satisfied by the final frame (or earlier with a stable and persistent completed state).
    \item For first-person camera view: no other agent should be performing the action.
\end{itemize}

\item \textbf{Realism and Physical Consistency}
\begin{itemize}
    \item The outcome must be grounded in objects present in earlier frames.
    \item If video quality is too poor to determine completion reliably, the task is \textbf{NOT} successful.
\end{itemize}

\end{enumerate}

\vspace{0.5em}
\textbf{Decision Rule}

Answer \textbf{Yes} only if all of the following are true:
\begin{itemize}
    \item The target agent performs the required action
    \item The action is fully completed by the final frame
    \item The action is physically realistic and consistent
\end{itemize}

If any condition fails, answer \textbf{No}.

\vspace{0.5em}
\textbf{Output Format}

Return exactly one word: \textbf{Yes} or \textbf{No}.  
Do not include any explanation or additional text.

\end{tcolorbox}
\caption{Prompt used for task reward during training.}
\label{fig:task_prompt}
\end{figure}

\begin{figure}[t]
\centering
\begin{tcolorbox}[title=Visual Quality and Temporal Consistency Prompt]

You are given generated video frames in correct temporal order (frames 0..N-1).

You are judging whether this temporally ordered video is successful in terms of visual quality and temporal consistency.

Analyze the video carefully and reason strictly from visible evidence.  
Do not infer hidden causes. Do not guess missing frames. Do not speculate beyond what is visible.

\vspace{0.5em}
\textbf{Evaluation Criteria}

\begin{enumerate}[leftmargin=1.5em]

\item \textbf{Visual Quality}
\begin{itemize}
    \item Frames should be clear, coherent, and stable.
    \item Severe blur, flicker, distortions, broken rendering, or major artifacts mean the video is \textbf{NOT} successful.
\end{itemize}

\item \textbf{Temporal Consistency}
\begin{itemize}
    \item Motion and state changes should be smooth and physically coherent over time.
    \item No teleportation, popping, identity instability, discontinuous motion, or implausible changes.
\end{itemize}

\item \textbf{Reliability of Evidence}
\begin{itemize}
    \item Judge only from visible evidence in the frames.
    \item If frames are too unclear to assess reliably, the video is \textbf{NOT} successful.
\end{itemize}

\item \textbf{Consistency with Initial Frame}
\begin{itemize}
    \item The video must remain consistent in style and quality with the first frame.
\end{itemize}

\end{enumerate}

\vspace{0.5em}
\textbf{Decision Rule}

Answer \textbf{Yes} only if all of the following are true:
\begin{itemize}
    \item Visual quality is acceptable overall without severe artifacts
    \item Temporal consistency is acceptable without severe continuity failures
    \item Frames are clear enough for reliable judgment
\end{itemize}

If any condition fails, answer \textbf{No}.

\vspace{0.5em}
\textbf{Output Format}

Return exactly one word: \textbf{Yes} or \textbf{No}.  
Do not include any explanation or additional text.

\end{tcolorbox}
\caption{Prompt used for the consistency reward during training.}
\label{fig:consistency_prompt}
\end{figure}

\subsection{Further Experiments}

We compared a per-step distillation reward with the trajectory-level distillation reward in Eq.~\ref{eq:distillation_reward} and found only minor differences in final performance. We therefore use the trajectory-level form in the main experiments for simplicity. We also investigated sharing weights between the Executor and Demonstrator. Across hyperparameter settings and EMA schedules, using the Executor weights as the Demonstrator led to unstable training, so all main results use a fixed Demonstrator.

\subsection{Evaluation Denominators}
\label{sec:evaluation_denominators}

VLM-based evaluation can occasionally fail because of malformed outputs or API errors. Main-text scores are computed over valid VLM judgments only. Tab.~\ref{tab:worldtasks_eval_denominators} reports the success counts, valid denominators, and failure rates for the \benchname comparison in Tab.~\ref{tab:world_tasks_bench_results}. This makes the amount of discarded data explicit; the largest observed failure rate is $3.0\%$.

\begin{table*}[t]
\caption{\textbf{Valid VLM-evaluation denominators on \benchname.} Each cell reports score (successes/valid judgments; failure rate), where the failure rate is computed against the 200-sample benchmark.}
\label{tab:worldtasks_eval_denominators}
\centering
\scriptsize
\setlength{\tabcolsep}{3pt}
\renewcommand{\arraystretch}{1.12}
\begin{tabularx}{\textwidth}{@{}>{\raggedright\arraybackslash}p{0.25\textwidth}YYYY@{}}
\toprule
\textbf{Method} &
\textbf{Task} $\uparrow$ &
\textbf{Agent} $\uparrow$ &
\textbf{Consist.} $\uparrow$ &
\textbf{Avg.} $\uparrow$ \\
\midrule
HY1.5
& 0.464 (91/196; 2.0\%)
& 0.540 (108/200; 0.0\%)
& 0.780 (156/200; 0.0\%)
& 0.597 (117/196; 2.0\%) \\
\rowcolor{gray!10}
HY1.5+\method$^{*}$
& 0.574 (113/197; 1.5\%)
& 0.630 (126/200; 0.0\%)
& 0.828 (164/198; 1.0\%)
& 0.673 (134/199; 0.5\%) \\
\midrule
LTX-2
& 0.315 (63/200; 0.0\%)
& 0.395 (79/200; 0.0\%)
& 0.690 (138/200; 0.0\%)
& 0.467 (93/199; 0.5\%) \\
LTX-2+SFT
& 0.292 (57/195; 2.5\%)
& 0.389 (77/198; 1.0\%)
& 0.682 (135/198; 1.0\%)
& 0.454 (89/196; 2.0\%) \\
\rowcolor{gray!10}
LTX-2+\method$^{*}$
& 0.452 (90/199; 0.5\%)
& 0.500 (100/200; 0.0\%)
& 0.693 (138/199; 0.5\%)
& 0.548 (109/199; 0.5\%) \\
\midrule
LTX-2 (8-Step)
& 0.285 (57/200; 0.0\%)
& 0.391 (77/197; 1.5\%)
& 0.694 (136/196; 2.0\%)
& 0.455 (91/200; 0.0\%) \\
LTX-2 (8-Step)+VLM
& 0.495 (99/200; 0.0\%)
& 0.572 (111/194; 3.0\%)
& 0.732 (145/198; 1.0\%)
& 0.598 (119/199; 0.5\%) \\
\rowcolor{gray!10}
LTX-2 (8-Step)+\method
& 0.605 (121/200; 0.0\%)
& 0.691 (134/194; 3.0\%)
& 0.882 (172/195; 2.5\%)
& 0.726 (143/197; 1.5\%) \\
\bottomrule
\end{tabularx}
\end{table*}

\subsection{VBVR Evaluation}
\label{sec:appendix_vbvr}

The VBVR tasks are substantially out of distribution for our setting because they are longer and more abstract than the short task instructions in \datasetname. We therefore evaluate the vanilla model, the \method-trained model, and a prompt-rewrite variant that converts the long benchmark query into a shorter task instruction before generation.
%
%
%
%
\begin{table}[t]
\centering
\small
\setlength{\tabcolsep}{6pt}
\begin{tabular}{lccccccc}
\toprule
\textbf{Model} & \textbf{Mean} $\uparrow$ & \textbf{Abstr.} $\uparrow$ & \textbf{Categ.} $\uparrow$ & \textbf{Navig.} $\uparrow$ & \textbf{Perc.} $\uparrow$ & \textbf{Physics} $\uparrow$ & \textbf{Transform.} $\uparrow$ \\
\midrule

LTX-2 
& 0.599  
& 0.619 & \textbf{0.623} & 0.596 & 0.611 & 0.574 & 0.573 \\

LTX-2+\method 
& \underline{0.613}  
& \underline{0.640} & 0.620 & \textbf{0.603} & \underline{0.630} & \textbf{0.580} & \underline{0.619} \\

LTX-2+\method$^{*}$ 
& \textbf{0.622}  
& \textbf{0.658} & \underline{0.621} & \textbf{0.603} & \textbf{0.662} & \underline{0.579} & \textbf{0.621} \\

\bottomrule
\end{tabular}
\caption{VBVR evaluation results across models and categories (250 samples each). $^{*}$ indicates the prompt-rewrite variant.}
\label{tab:vbvr_results}
\end{table}

\subsection{Further Details on \datasetname}
\label{sec:appendix_further_details_dataset}
\subsubsection{Dataset Filtering}
\label{sec:appendix_dataset_filtering}
We construct the dataset from pre-extracted images. We first remove incomplete or already processed entries, and then apply an image-quality filter. This filter rejects frames that are excessively blurry, underexposed, overexposed, or nearly empty. Concretely, we compute the variance of the Laplacian as a blur indicator, the mean luminance to detect overly dark or bright images, and the fraction of near-black and near-white pixels to remove degenerate frames. In our main setup, we use thresholds of $\texttt{min\_laplacian\_var}=12.0$, $\texttt{min\_mean\_luma}=20.0$, $\texttt{max\_mean\_luma}=235.0$, $\texttt{max\_black\_ratio}=0.85$, and $\texttt{max\_white\_ratio}=0.85$.

To further improve visual quality, we rank the surviving frames with an aesthetic score based on a CLIP-based scoring function and keep only the top $90\%$ of samples according to the combined quality score. We also apply a vision-language quality screening step that discards frames deemed unsuitable for agent-based video generation. 

The VLM is prompted to assess whether an image is appropriate for agent-based video generation based on the prompt given in Box~\ref{fig:vlm_filter_prompt}.

\begin{figure}[t]
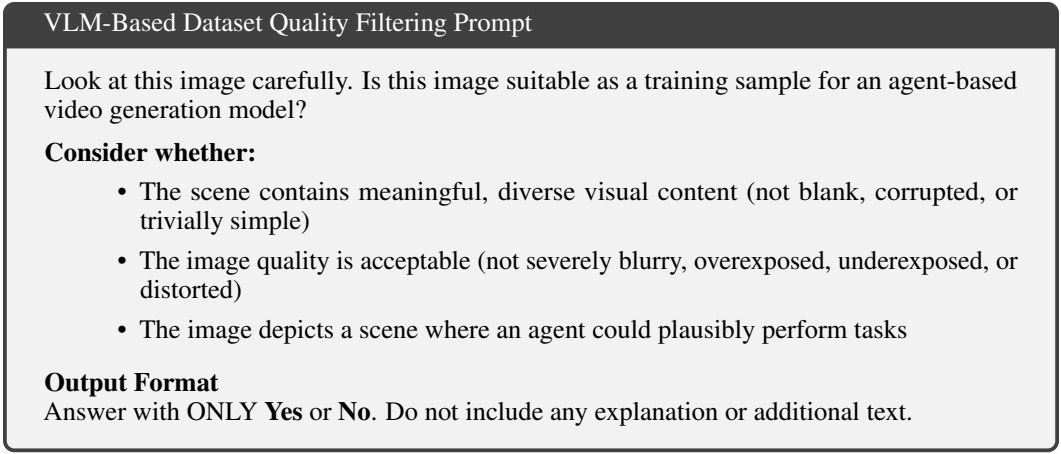

\centering
\begin{tcolorbox}[title=VLM-Based Dataset Quality Filtering Prompt]

Look at this image carefully. Is this image suitable as a training sample for an agent-based video generation model?

\vspace{0.5em}
\textbf{Consider whether:}
\begin{itemize}
    \item The scene contains meaningful, diverse visual content (not blank, corrupted, or trivially simple)
    \item The image quality is acceptable (not severely blurry, overexposed, underexposed, or distorted)
    \item The image depicts a scene where an agent could plausibly perform tasks
\end{itemize}

\vspace{0.5em}
\textbf{Output Format}

Answer with ONLY \textbf{Yes} or \textbf{No}.  
Do not include any explanation or additional text.

\end{tcolorbox}
\caption{Prompt used for VLM-based semantic filtering of dataset images.}
\label{fig:vlm_filter_prompt}
\end{figure}

Samples for which the VLM responds negatively are discarded. This additional semantic filtering step complements the low-level image quality criteria by removing visually valid but uninformative or non-actionable scenes, resulting in a dataset that is both visually and semantically suitable for downstream task-conditioned video generation.

\subsubsection{Example Task and Solution Prompts from \datasetname}
\label{sec:appendix_dataset_examples}

To qualitatively illustrate the structure of \datasetname, we present four representative samples below. Each example contains the first frame together with the first two task prompts and their corresponding descriptive solution prompts. Examples are shown in Fig.~\ref{fig:dataset_examples_1} and Fig.~\ref{fig:dataset_examples_2}.

\vspace{0.5em}

\noindent
\begin{figure}[t]
\centering
\begin{minipage}[t]{0.49\linewidth}
\vspace{0pt}
\begin{tcolorbox}[
  colback=white,
  colframe=black!15,
  boxrule=0.5pt,
  arc=2pt,
  left=6pt,right=6pt,top=6pt,bottom=6pt
]
\centering
\includegraphics[width=\linewidth]{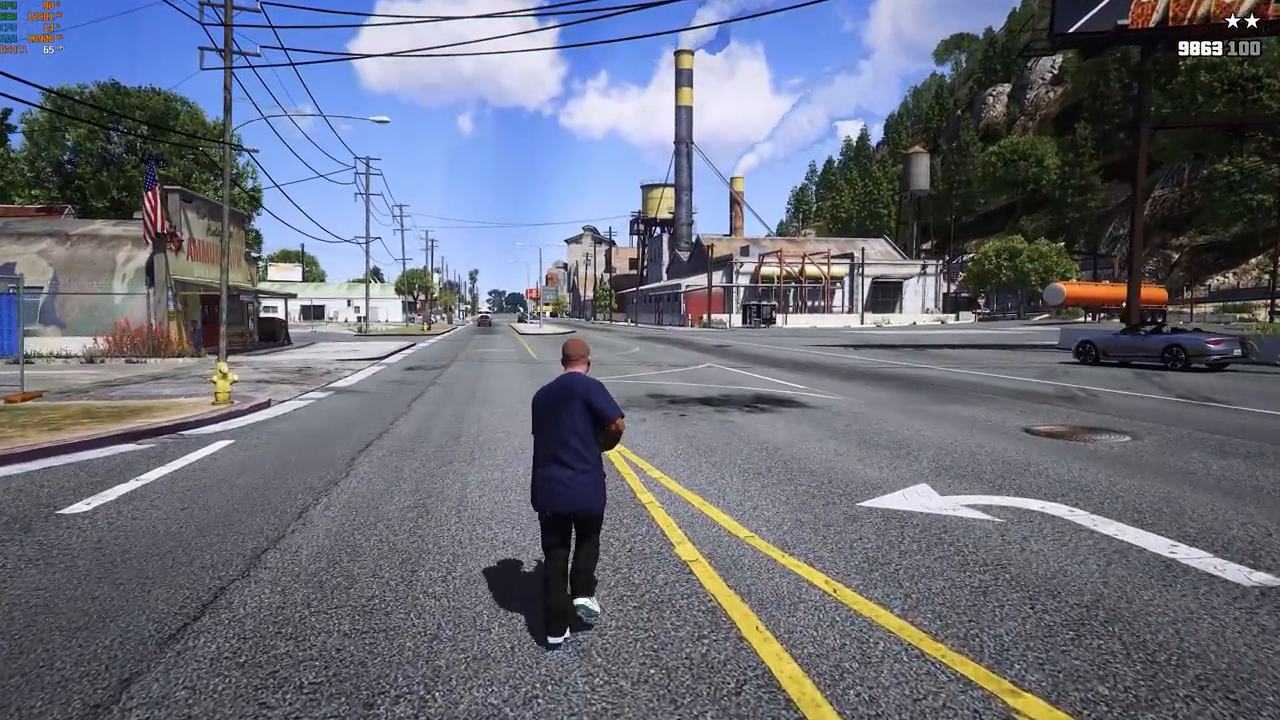}

\vspace{0.8em}

\textbf{Sample ID:} \texttt{100\_4}

\vspace{0.6em}

\begin{enumerate}[leftmargin=1.4em,itemsep=0.8em,topsep=0.2em]
    \item \textbf{Task 1:} [Man in blue shirt]: Step onto the yellow lane marking and stop exactly at the white arrow's tip.

    \textbf{Description 1:} The man in the blue shirt begins walking forward along the center of the road, his feet deliberately stepping onto the double yellow lane marking, and continues moving straight ahead until he reaches the tip of the white directional arrow painted on the asphalt, where he halts and stands still.

    \item \textbf{Task 2:} [Person in blue shirt]: Move forward to the nearest building. 

    \textbf{Description 2:} The person in the blue shirt begins walking forward along the center of the road, maintaining a steady pace toward the building on the left side of the street, their body oriented directly ahead as they cross the yellow double lines; after a few steps, they continue moving forward until they reach the sidewalk adjacent to the building, then they halt beside the American flag mounted on the building's facade, coming to a complete stop with their feet planted on the pavement.
\end{enumerate}
\end{tcolorbox}
\end{minipage}
\hfill
\begin{minipage}[t]{0.49\linewidth}
\vspace{0pt}
\begin{tcolorbox}[
  colback=white,
  colframe=black!15,
  boxrule=0.5pt,
  arc=2pt,
  left=6pt,right=6pt,top=6pt,bottom=6pt
]
\centering
\includegraphics[width=\linewidth]{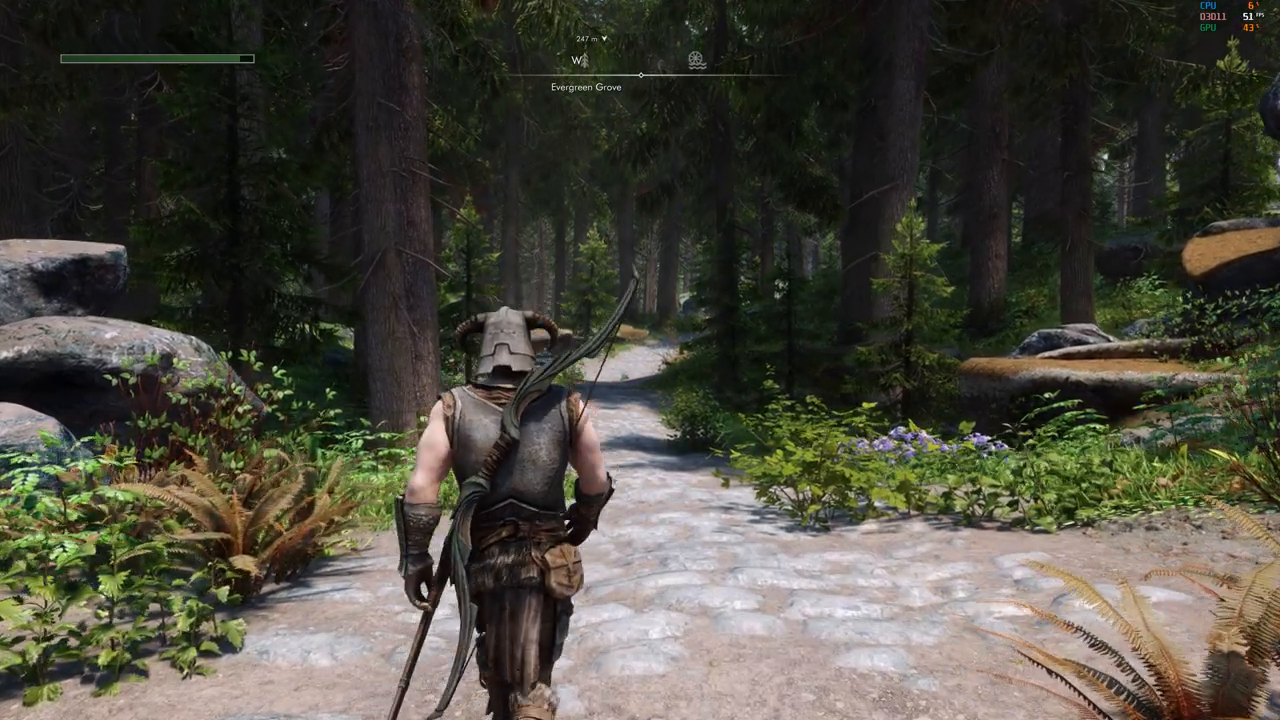}

\vspace{0.8em}

\textbf{Sample ID:} \texttt{145\_4}

\vspace{0.6em}

\begin{enumerate}[leftmargin=1.4em,itemsep=0.8em,topsep=0.2em]
    \item \textbf{Task 1:} [Character with horned helmet]: Use the bow to aim at the tree trunk directly ahead.

    \textbf{Description 1:} The character with the horned helmet slowly turns their upper body toward the tree trunk directly ahead, simultaneously drawing the bowstring back with their right hand while keeping their left hand steady on the bow's grip, their gaze fixed on the target as the bowstring tenses and the arrow nocks align with the trunk.
    
    \item \textbf{Task 2:} [Character with horned helmet]: Move to the largest boulder and stop beside its left edge.

    \textbf{Description 2:} The character with the horned helmet begins walking forward along the stone path, their body oriented toward the largest boulder visible to the left, and after a few steps, they decelerate, shifting their weight slightly as they turn their head to the left to align their gaze with the boulder's edge, then halt precisely beside its left side, their right hand resting on their hip while their left hand remains near the hilt of their weapon.
\end{enumerate}
\end{tcolorbox}
\end{minipage}
\caption{Two representative samples from \datasetname. Each sample includes an initial frame, task prompts, and corresponding descriptive solutions.}
\label{fig:dataset_examples_1}

\end{figure}

\begin{figure}[t]
\centering
\begin{minipage}[t]{0.49\linewidth}
\vspace{0pt}
\begin{tcolorbox}[
  colback=white,
  colframe=black!15,
  boxrule=0.5pt,
  arc=2pt,
  left=6pt,right=6pt,top=6pt,bottom=6pt
]
\centering
\includegraphics[width=\linewidth]{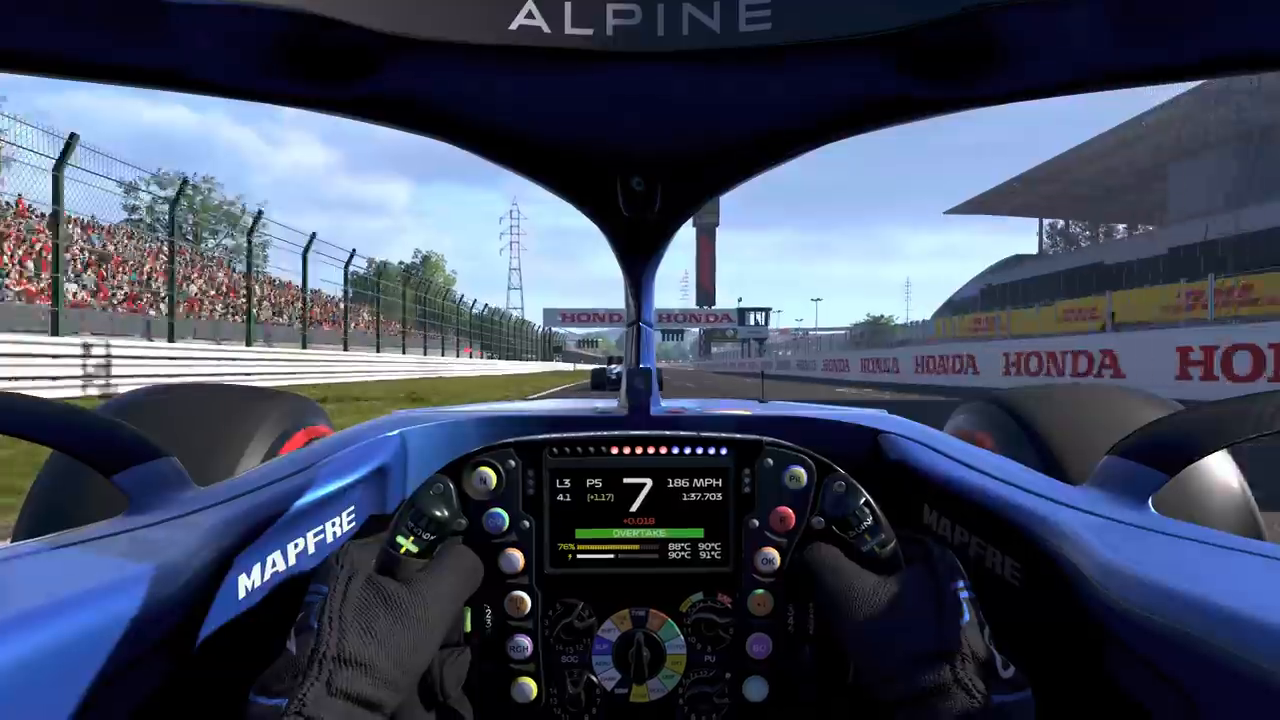}

\vspace{0.8em}

\textbf{Sample ID:} \texttt{6888\_1}

\vspace{0.6em}

\begin{enumerate}[leftmargin=1.4em,itemsep=0.8em,topsep=0.2em]
    \item \textbf{Task 1:} [Driver in racing suit]: Press the red button on the steering wheel's right side.

    \textbf{Description 1:} The driver's right hand, clad in a black racing glove, moves slightly forward and inward, pressing the red button located on the right side of the steering wheel, while the left hand remains steady on the left side of the wheel, and the vehicle continues forward along the track with the dashboard displaying 186 MPH and an overtaking indicator active.

    \item \textbf{Task 2:} [First-person view]: Align the car's front bumper with the white track curb ahead.

    \textbf{Description 2:} The driver's hands grip the steering wheel firmly, thumbs pressing the paddle shifters while the left hand subtly adjusts its position to maintain control; simultaneously, the right hand makes a slight inward rotation of the wheel to initiate a gentle steering correction toward the white track curb ahead, and the car's front bumper begins to approach the curb as the vehicle decelerates slightly, aligning its front edge with the curb's edge while the dashboard display updates to reflect the new position and speed.
\end{enumerate}
\end{tcolorbox}
\end{minipage}
\hfill
\begin{minipage}[t]{0.49\linewidth}
\vspace{0pt}
\begin{tcolorbox}[
  colback=white,
  colframe=black!15,
  boxrule=0.5pt,
  arc=2pt,
  left=6pt,right=6pt,top=6pt,bottom=6pt
]
\centering
\includegraphics[width=\linewidth]{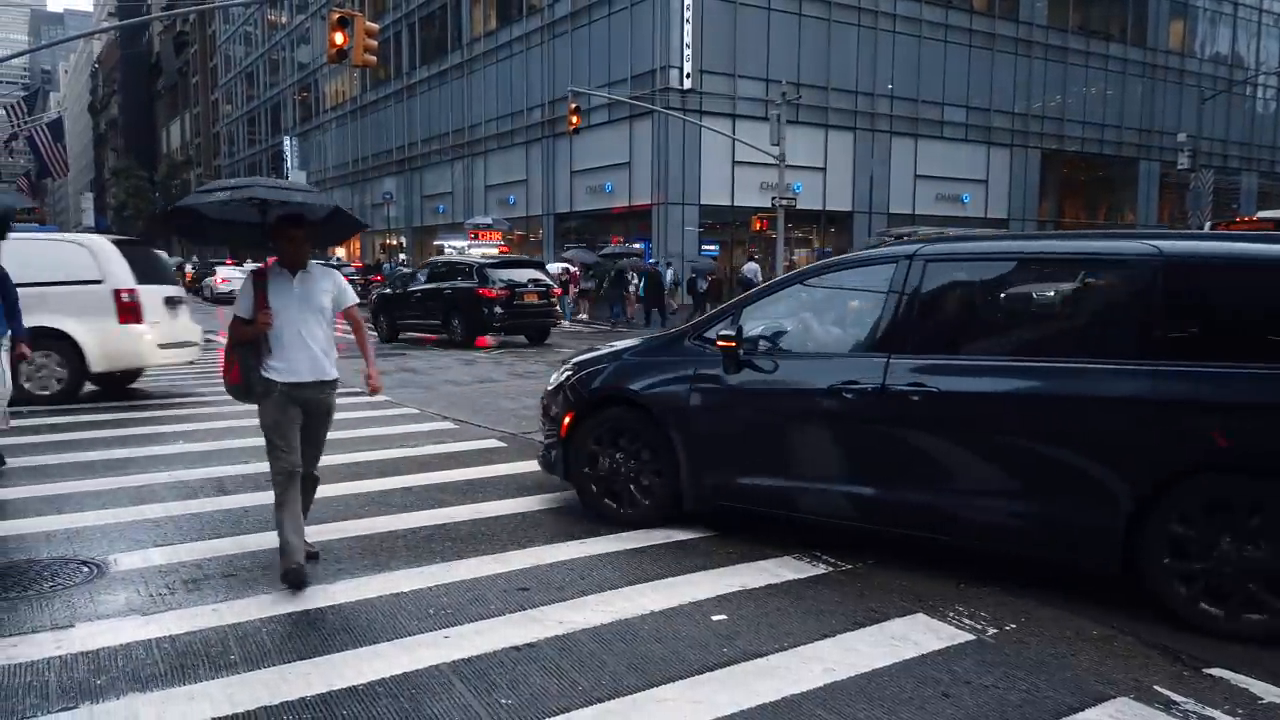}

\vspace{0.8em}

\textbf{Sample ID:} \texttt{7637\_6}

\vspace{0.6em}

\begin{enumerate}[leftmargin=1.4em,itemsep=0.8em,topsep=0.2em]
    \item \textbf{Task 1:} [Man holding black umbrella]: Step off the crosswalk and hand the umbrella to the sidewalk curb.

    \textbf{Description 1:} The man holding the black umbrella continues walking forward, stepping off the crosswalk onto the sidewalk, then lowers his arm and extends his hand toward the curb, releasing the umbrella to rest against the sidewalk edge.

    \item \textbf{Task 2:} [Black minivan]: Align its front bumper with the white pedestrian lane marking.

    \textbf{Description 2:} The black minivan advances forward while maintaining its current trajectory, its front bumper gradually moving closer to the white pedestrian lane marking on the asphalt, adjusting its position as it proceeds along the crosswalk.
\end{enumerate}
\end{tcolorbox}
\end{minipage}
\caption{Two representative samples from \datasetname. Each sample includes an initial frame, task prompts, and corresponding descriptive solutions.}
\label{fig:dataset_examples_2}

\end{figure}

\subsection{Evaluation Prompts for \benchname}
\label{sec:evaluation_prompts}
In this section, we present the evaluation prompts used in \benchname to assess generated videos along three complementary dimensions. The first prompt evaluates whether the instructed task is successfully completed, focusing strictly on end-state correctness. The second prompt verifies correct agent attribution, ensuring that the intended actor performs the specified action. The third prompt measures physical realism and temporal consistency, capturing whether the video exhibits plausible motion and coherent dynamics. Together, these prompts provide a structured and binary evaluation framework that isolates task success, agent correctness, and physical validity.

\begin{figure}[t]
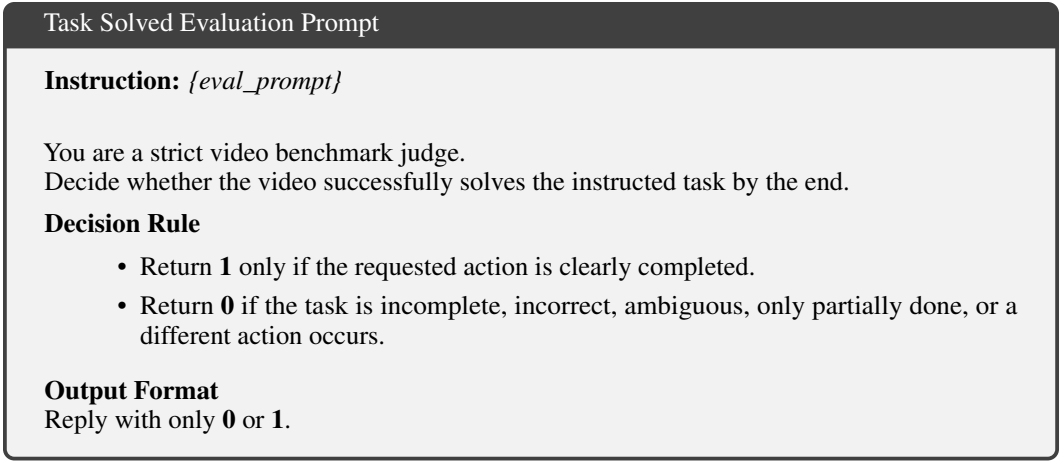

\centering
\begin{tcolorbox}[title=Task Solved Evaluation Prompt]

\textbf{Instruction:} \textit{\{eval\_prompt\}} \\

\vspace{0.5em}

You are a strict video benchmark judge.

Decide whether the video successfully solves the instructed task by the end.

\vspace{0.5em}
\textbf{Decision Rule}

\begin{itemize}
    \item Return \textbf{1} only if the requested action is clearly completed.
    \item Return \textbf{0} if the task is incomplete, incorrect, ambiguous, only partially done, or a different action occurs.
\end{itemize}

\vspace{0.5em}
\textbf{Output Format}

Reply with only \textbf{0} or \textbf{1}.

\end{tcolorbox}
\caption{Prompt used to evaluate whether a generated video successfully completes the instructed task.}
\label{fig:task_solved_prompt}
\end{figure}

\begin{figure}[t]
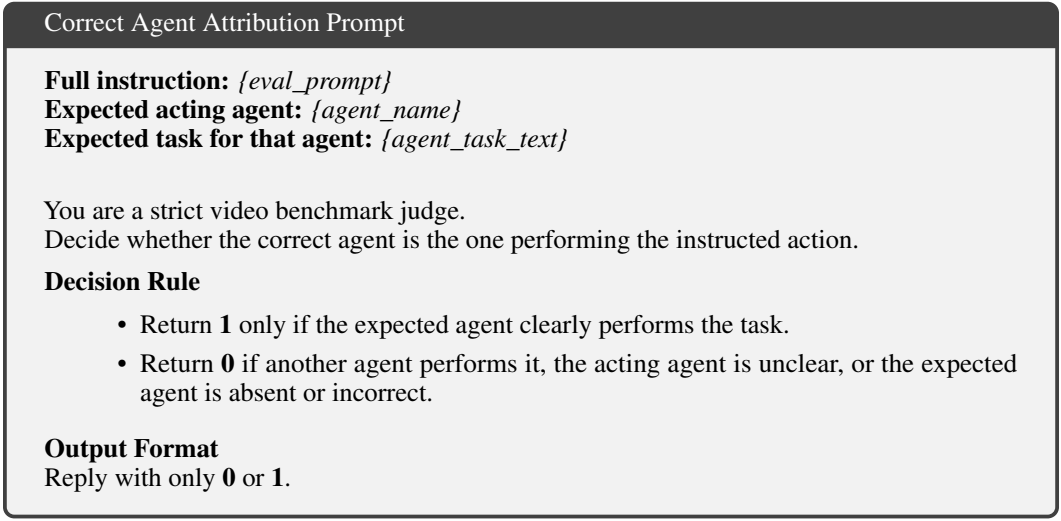

\centering
\begin{tcolorbox}[title=Correct Agent Attribution Prompt]

\textbf{Full instruction:} \textit{\{eval\_prompt\}} \\
\textbf{Expected acting agent:} \textit{\{agent\_name\}} \\
\textbf{Expected task for that agent:} \textit{\{agent\_task\_text\}} \\

\vspace{0.5em}

You are a strict video benchmark judge.

Decide whether the correct agent is the one performing the instructed action.

\vspace{0.5em}
\textbf{Decision Rule}

\begin{itemize}
    \item Return \textbf{1} only if the expected agent clearly performs the task.
    \item Return \textbf{0} if another agent performs it, the acting agent is unclear, or the expected agent is absent or incorrect.
\end{itemize}

\vspace{0.5em}
\textbf{Output Format}

Reply with only \textbf{0} or \textbf{1}.

\end{tcolorbox}
\caption{Prompt used to verify that the correct agent performs the instructed action.}
\label{fig:correct_agent_prompt}
\end{figure}

\begin{figure}[t]
\centering
\begin{tcolorbox}[title=Physical Realism and Temporal Consistency Prompt]

\vspace{0.5em}

You are a strict video benchmark judge.

Decide whether the video shows physically realistic execution.

\vspace{0.5em}
\textbf{Decision Rule}

\begin{itemize}
    \item Return \textbf{1} only if motion, contact, timing, object interactions, and scene dynamics are physically plausible and temporally coherent.
    \item Return \textbf{0} if there is teleportation, impossible motion, severe temporal inconsistency, broken dynamics, identity drift, or obvious non-physical behavior.
\end{itemize}

\vspace{0.5em}
\textbf{Output Format}

Reply with only \textbf{0} or \textbf{1}.

\end{tcolorbox}
\caption{Prompt used to evaluate physical realism and temporal consistency of generated videos.}
\label{fig:physical_realism_prompt}
\end{figure}

\subsection{Alternating Training Algorithm}
We present the alternating training procedure  (Alg.~\ref{alg:1}).
\begin{algorithm}[ht!]
\caption{GRPO/AWM with Demonstrator Anchoring}
\label{alg:alt_grpo_demo_simple}
\begin{algorithmic}[1]
\Require pretrained base model $f_0$ with velocity field $v_0$
\Require Executor sampler $p_{\theta}(\cdot \mid c_{\mathrm{E}})$ and velocity field $v_{\theta}$
\Require Demonstrator sampler $p_{\phi}(\cdot \mid c_{\mathrm{D}})$ and velocity field $v_{\phi}$; fixed for the main \method setting
\Require VLM-generated condition pairs $\mathcal{S} = \{(c_{\mathrm{E}}, c_{\mathrm{D}})\}$
\Require reward weights $\lambda_{\mathrm{task}},\lambda_{\mathrm{distill}}$
\Require Executor anchor coefficient $\beta_d$ and optional Demonstrator anchor coefficient $\beta_{\phi}$
\Require optional alternation period $N$; set $N=0$ to keep the Demonstrator fixed
\For{iteration $e = 1,2,\dots$}
    \State sample $(c_{\mathrm{E}}, c_{\mathrm{D}}) \sim \mathcal{S}$

    \If{$N > 0$ \textbf{and} $e \bmod N = 0$}
        \State \textbf{Optional Demonstrator round}
        \State sample rollout group $\{\tau_i\}_{i=1}^{G} \sim p_{\phi}(\cdot \mid c_{\mathrm{D}})$
        \State compute task rewards $r_{\mathrm{task}}(\tau_i; \mathcal{I},\mathcal{T})$
        \State compute group-relative RL loss $\mathcal{L}_{\mathrm{RL}}(\phi)$ from the task rewards
        \State compute base-model anchor loss $\mathcal{L}_{\mathrm{base}}$ against $f_0$
        \State update $\phi$ by minimizing $\mathcal{L}_{\mathrm{RL}}(\phi)+\beta_{\phi}\mathcal{L}_{\mathrm{base}}$
    \Else
        \State \textbf{Executor round}
        \State sample rollout group $\{\tau_i\}_{i=1}^{G} \sim p_{\theta}(\cdot \mid c_{\mathrm{E}})$
        \State compute task rewards $r_{\mathrm{task}}(\tau_i; \mathcal{I},\mathcal{T})$
        \State compute distillation rewards $r_{\mathrm{distill}}(\tau_i)$ using Eq.~\ref{eq:distillation_reward}
        \State set $R_i=\lambda_{\mathrm{task}}r_{\mathrm{task}}(\tau_i)+\lambda_{\mathrm{distill}}r_{\mathrm{distill}}(\tau_i)$
        \State compute group-relative RL loss $\mathcal{L}_{\mathrm{RL}}(\theta)$ from $\{R_i\}_{i=1}^{G}$
        \State compute Demonstrator anchor loss $\mathcal{L}_{\mathrm{anchor}}$ using Eq.~\ref{eq:anchor_loss}
        \State update $\theta$ by minimizing $\mathcal{L}_{\mathrm{RL}}(\theta)+\beta_d\mathcal{L}_{\mathrm{anchor}}$
    \EndIf
\EndFor

\State \textbf{Evaluation:} use $p_{\theta}(\cdot \mid c_{\mathrm{E}})$
\end{algorithmic}
\label{alg:1}
\end{algorithm}

\subsection{Theoretical Background}
\label{sec:theoretical_background}

\paragraph{Group-relative policy optimization.}
We optimize the student model using a group-relative reinforcement learning objective inspired by GRPO~\cite{Shao2024DeepSeekMathPT}. For each task instruction $\mathcal{T}$, we sample $G$ trajectories
\[
\tau_1,\ldots,\tau_G \sim p_{\theta}(\cdot \mid \mathcal{T}),
\]
and compute their rewards $r(\tau_i)$. These rewards are normalized within the group to produce relative advantages
\[
A_i
=
\frac{r(\tau_i)-\mu_r}{\sigma_r+\varepsilon},
\quad
\mu_r=\frac{1}{G}\sum_{j=1}^{G} r(\tau_j),
\quad
\sigma_r^2=\frac{1}{G}\sum_{j=1}^{G}(r(\tau_j)-\mu_r)^2.
\]
The resulting objective is
\[
\mathcal{L}_{\mathrm{GRPO}}
=
-
\mathbb{E}_{\mathcal{T}}
\left[
\frac{1}{G}\sum_{i=1}^{G}
A_i \log p_{\theta}(\tau_i \mid \mathcal{T})
\right].
\]
This formulation reinforces trajectories that outperform their peers on the same task while suppressing weaker ones. Unlike standard distillation, it enables improvements beyond the teacher whenever the reward function favors better solutions.

\paragraph{FlowGRPO.}
Flow matching models learn a continuous transport from noise $x_0 \sim \mathcal{N}(0,I)$ to data $x_1 \sim p_{\mathrm{data}}$ via
\[
x_t = (1-t)x_0 + tx_1,
\]
and are trained with the objective
\[
\mathcal{L}_{\mathrm{FM}}(\theta)
=
\mathbb{E}_{x_0,x_1,t}
\left[
\| v_\theta(x_t,t,c) - (x_1 - x_0) \|^2
\right].
\]
This enables deterministic sampling through the ODE
\[
dx_t = v_\theta(x_t,t)\,dt.
\]

Flow-GRPO~\cite{Liu2025FlowGRPO} extends this framework by casting denoising as a multi-step MDP. Here the subscript $t-1$ denotes the next state in the discrete reverse sampler, not the continuous flow-time convention above. The state, action, and policy are defined as
\[
s_t = (c,t,x_t),
\quad
a_t = x_{t-1},
\quad
\pi_\theta(a_t \mid s_t) = p_\theta(x_{t-1}\mid x_t,c).
\]
To introduce exploration, the deterministic flow is converted into an SDE:
\[
dx_t
=
\left(
v_t(x_t)
-
\frac{\sigma_t^2}{2}\nabla \log p_t(x_t)
\right)dt
+
\sigma_t\,dw.
\]

The model is then optimized using a clipped GRPO-style objective:
\begin{align*}
J_{\mathrm{Flow\text{-}GRPO}}(\theta)
&=
\mathbb{E}
\Biggl[
\frac{1}{G}
\sum_{i=1}^{G}
\frac{1}{T}
\sum_{t=0}^{T-1}
\Bigl(
\min\!\left(
r_t^i(\theta)\hat{A}_t^i,
\right. \\
&\qquad\qquad\left.
\operatorname{clip}(r_t^i(\theta),1-\epsilon,1+\epsilon)\hat{A}_t^i
\right)
-
\beta
D_{\mathrm{KL}}(\pi_\theta \| \pi_{\mathrm{ref}})
\Bigr)
\Biggr],
\end{align*}
where
\[
r_t^i(\theta)
=
\frac{p_\theta(x_{t-1}^i \mid x_t^i,c)}
{p_{\theta_{\mathrm{old}}}(x_{t-1}^i \mid x_t^i,c)},
\]
and
\[
\hat{A}_t^i
=
\frac{
R(x_1^i,c)
-
\operatorname{mean}(\{R(x_1^j,c)\}_{j=1}^{G})
}{
\operatorname{std}(\{R(x_1^j,c)\}_{j=1}^{G})
}.
\]

\paragraph{Advantage Weighted Matching.}
Advantage Weighted Matching (AWM)~\cite{Xue2025AdvantageWM} addresses a mismatch between diffusion-style reinforcement learning objectives and the original flow-matching training objective. Methods such as DDPO effectively optimize noisy reverse-step likelihoods, which increases variance and slows convergence.

AWM instead preserves the original flow-matching objective and incorporates rewards through advantage weighting. The prompt $c$ defines the state, and the final sample $x_1$ is treated as the action with policy
\[
\pi_\theta(x_1 \mid c).
\]
The sequence likelihood is approximated by the negative flow-matching loss:
\[
\log \hat{\pi}_\theta(x_1 \mid c)
\approx
-
\mathbb{E}_{t}
\left[
w(t)\| v_\theta(x_t,t,c) - (x_1 - x_0) \|^2
\right].
\]
This yields the likelihood ratio
\begin{align*}
\frac{\hat{\pi}_\theta(x_1 \mid c)}
{\hat{\pi}_{\theta_{\mathrm{old}}}(x_1 \mid c)}
&=
\exp\!\Biggl(
-
\mathbb{E}_{t}
\Bigl[
w(t)\| v_\theta(x_t,t,c) - (x_1 - x_0) \|^2 \\
&\hspace{7.5em}
-
w(t)\| v_{\theta_{\mathrm{old}}}(x_t,t,c) - (x_1 - x_0) \|^2
\Bigr]
\Biggr).
\end{align*}

The corresponding policy-gradient update is
\[
\nabla_\theta \log \hat{\pi}_\theta(x_1 \mid c)\, A
=
-
\nabla_\theta
\mathbb{E}_{t}
\left[
w(t)\| v_\theta(x_t,t,c) - (x_1 - x_0) \|^2
\right] A.
\]
Positive advantages reduce the flow-matching loss for high-reward samples, while negative advantages suppress low-reward ones. AWM further includes a velocity-space KL regularizer
\[
D_{\mathrm{KL}}
\approx
w(t)\| v_\theta(x_t,t,c) - v_{\mathrm{ref}}(x_t,t,c) \|^2,
\]
which stabilizes updates by constraining deviations from a reference model.

In contrast to Flow-GRPO, which introduces stochastic trajectory optimization, AWM directly aligns reinforcement learning with the original flow-matching objective, resulting in lower variance and improved training efficiency.

\subsection{Method Derivations and Proofs}
In this section we provide the derivations and proofs from Sec.~\ref{sec:method}.

\paragraph{Gradient Decomposition}
We begin with deriving the gradient decomposition in Eq.~\ref{eq:on_policy_gradient_decomposition}.

\begin{align*}
\mathcal{L}_{\mathrm{on}}
&=
\int p_{\theta}(\tau \mid c_{\mathrm{E}})\, C_{\theta}(\tau)\, d\tau, \\
\nabla_{\theta} \mathcal{L}_{\mathrm{on}}
&=
\int \nabla_{\theta}
\big[ p_{\theta}(\tau \mid c_{\mathrm{E}})\, C_{\theta}(\tau) \big] d\tau \\
&=
\int \Big[
C_{\theta}(\tau)\nabla_{\theta}p_{\theta}(\tau \mid c_{\mathrm{E}})
+
p_{\theta}(\tau \mid c_{\mathrm{E}})\nabla_{\theta}C_{\theta}(\tau)
\Big] d\tau 
\quad \text{\small (product rule)}\\
&=
\int p_{\theta}(\tau \mid c_{\mathrm{E}})
C_{\theta}(\tau)
\nabla_{\theta}\log p_{\theta}(\tau \mid c_{\mathrm{E}})\, d\tau 
\quad \text{\small (score function trick)}\\
&\quad+
\int p_{\theta}(\tau \mid c_{\mathrm{E}})
\nabla_{\theta}C_{\theta}(\tau)\, d\tau \\
&=
\mathbb{E}_{\tau \sim p_{\theta}}
\big[
C_{\theta}(\tau)\nabla_{\theta}\log p_{\theta}(\tau \mid c_{\mathrm{E}})
\big]
+
\mathbb{E}_{\tau \sim p_{\theta}}
\big[
\nabla_{\theta}C_{\theta}(\tau)
\big]
\quad \text{\small (rewrite as expectations)}.
\end{align*}

\paragraph{Proposition}
We continue with the proof of Proposition~\ref{prop1}:

\begin{proof}
Let $x_t^\theta$ and $x_t^{\theta'}$ denote the student and teacher flows initialized from the same
$x_0\sim p_0$, so that $x_0^\theta=x_0^{\theta'}$. Define
$\Delta_t:=x_t^\theta-x_t^{\theta'}$. Then
\[
\frac{d}{dt}\Delta_t
=
v_\theta(x_t^\theta,t\mid c_{\mathrm E})
-
v_{\theta'}(x_t^{\theta'},t\mid c_{\mathrm D}).
\]
Adding and subtracting $v_{\theta'}(x_t^\theta,t\mid c_{\mathrm D})$ and using the $L$-Lipschitzness of
$v_{\theta'}(\cdot,t\mid c_{\mathrm D})$ gives
\[
\frac{d}{dt}\|\Delta_t\|
\le
\left\|
v_\theta(x_t^\theta,t\mid c_{\mathrm E})
-
v_{\theta'}(x_t^\theta,t\mid c_{\mathrm D})
\right\|
+
L\|\Delta_t\|.
\]
Since $\Delta_0=0$, Gr\"onwall's inequality implies
\[
\|\Delta_1\|
\le
e^L
\int_0^1
\left\|
v_\theta(x_t^\theta,t\mid c_{\mathrm E})
-
v_{\theta'}(x_t^\theta,t\mid c_{\mathrm D})
\right\|\,dt.
\]
Therefore, by Cauchy--Schwarz,
\[
\mathbb E_{x_0\sim p_0}\|\Delta_1\|^2
\le
e^{2L}
\mathbb E_{x_0\sim p_0}
\left[
\int_0^1
\left\|
v_\theta(x_t^\theta,t\mid c_{\mathrm E})
-
v_{\theta'}(x_t^\theta,t\mid c_{\mathrm D})
\right\|^2 dt
\right]
\le
e^{2L}\varepsilon^2.
\]
The shared initialization defines a valid coupling between the terminal laws
$p_\theta(x_1\mid c_{\mathrm E})$ and $p_{\theta'}(x_1\mid c_{\mathrm D})$. Hence
\[
W_2\!\left(
p_\theta(x_1\mid c_{\mathrm E}),
p_{\theta'}(x_1\mid c_{\mathrm D})
\right)
\le
\left(\mathbb E_{x_0\sim p_0}\|\Delta_1\|^2\right)^{1/2}
\le
e^L\varepsilon.
\]
\end{proof}